\documentclass[final,5p,times,twocolumn]{template/elsarticle}

\usepackage{bmpsize}
\usepackage{amsthm}
\usepackage[linesnumbered,ruled,vlined]{algorithm2e}
\usepackage{amsmath,amssymb,enumerate}
\usepackage{mathrsfs}
\usepackage{times}
\usepackage{multicol}
\usepackage{amsfonts}
\usepackage{textcomp}
\usepackage{balance}
\usepackage{multirow}
\usepackage{booktabs}
\usepackage{dblfloatfix} 
\usepackage{bbold}
\usepackage{makecell}
\usepackage{hhline} 
\usepackage{float} 

\usepackage{graphicx} 
\usepackage[T1]{fontenc}
\usepackage[usenames,dvipsnames,svgnames,table, xcdraw]{xcolor}
\usepackage[caption=false,font=footnotesize]{subfig}
\usepackage[utf8]{inputenc}

\usepackage{hyperref}
\hypersetup{
  colorlinks=true,pdfpagemode=UseNone,citecolor=black,linkcolor=black,urlcolor=BrickRed
}

\usepackage{caption}
\newtheorem{theorem}{Theorem}
\newtheorem{proposition}{Proposition}
\newtheorem{remark}{Remark}

\newtheorem{definition}{Definition}

\newcommand{\velodyneN}{\textit{32-Beam Velodyne ULTRA Puck LiDAR}}

\newcommand{\xaxisN}{$x$-axis}

\newcommand{\xaxis}{$x$-axis~}
\newcommand{\yaxis}{$y$-axis~}

\newcommand{\comment}[1]{}

\def\real{\mathbb{R}}

\DeclareDocumentCommand{\RI}{ O{} O{} }{\mathcal{RI}_{#1}^{#2}}

\DeclareDocumentCommand{\td}{ O{} }{\tilde{#1}}
\newcommand{\transpose}{\mathsf{T}}

\DeclareDocumentCommand{\asin}{ O{} }{\sin^{-1}(#1)}
\DeclareDocumentCommand{\acos}{ O{} }{\cos^{-1}(#1)}
\DeclareDocumentCommand{\atan}{ O{} }{\tan^{-1}(#1)}

\DeclareDocumentCommand{\vector}{ O{} }{\mathrm{vec}(#1)}
\DeclareDocumentCommand{\zeros}{ O{} }{\textbf{0}_{#1}}
\DeclareDocumentCommand{\pre}{ O{} O{} }{{}_{#1}^{#2}}

\DeclareMathOperator*{\argmin}{arg\,min}

\newcommand{\Ecal}{\mathcal{E}}
\newcommand{\Gcal}{\mathcal{G}}

\newcommand{\Ocal}{\mathcal{O}}
\newcommand{\Pcal}{\mathcal{P}}

\DeclareDocumentCommand{\A}{ O{} O{} }{\textbf{A}_{#1}^{#2}}
\DeclareDocumentCommand{\H}{ O{} O{} }{\textbf{H}_{#1}^{#2}}
\DeclareDocumentCommand{\I}{ O{} O{} }{\textbf{I}_{#1}^{#2}}
\DeclareDocumentCommand{\L}{ O{} O{} }{\textbf{L}_{#1}^{#2}}
\DeclareDocumentCommand{\M}{ O{} O{} }{\textbf{M}_{#1}^{#2}}
\DeclareDocumentCommand{\N}{ O{} O{} }{\textbf{N}_{#1}^{#2}}
\DeclareDocumentCommand{\O}{ O{} O{} }{\textbf{O}_{#1}^{#2}}
\DeclareDocumentCommand{\P}{ O{} O{} }{\textbf{P}_{#1}^{#2}}
\DeclareDocumentCommand{\Q}{ O{} O{} }{\textbf{Q}_{#1}^{#2}}
\DeclareDocumentCommand{\R}{ O{} O{} }{\textbf{R}_{#1}^{#2}}
\DeclareDocumentCommand{\T}{ O{} O{} }{\textbf{T}_{#1}^{#2}}
\DeclareDocumentCommand{\U}{ O{} O{} }{\textbf{U}_{#1}^{#2}}
\DeclareDocumentCommand{\V}{ O{} O{} }{\textbf{V}_{#1}^{#2}}
\DeclareDocumentCommand{\X}{ O{} O{} }{\textbf{X}_{#1}^{#2}}
\DeclareDocumentCommand{\Y}{ O{} O{} }{\textbf{Y}_{#1}^{#2}}
\DeclareDocumentCommand{\Z}{ O{} O{} }{\textbf{Z}_{#1}^{#2}}

\DeclareDocumentCommand{\e}{ O{} O{} }{\textbf{e}_{#1}^{#2}}
\DeclareDocumentCommand{\n}{ O{} O{} }{\textbf{n}_{#1}^{#2}}
\DeclareDocumentCommand{\o}{ O{} O{} }{\textbf{o}_{#1}^{#2}}
\DeclareDocumentCommand{\t}{ O{} O{} }{\textbf{t}_{#1}^{#2}}
\DeclareDocumentCommand{\p}{ O{} O{} }{\textbf{p}_{#1}^{#2}}
\DeclareDocumentCommand{\q}{ O{} O{} }{\textbf{q}_{#1}^{#2}}
\DeclareDocumentCommand{\r}{ O{} O{} }{\textbf{r}_{#1}^{#2}}
\DeclareDocumentCommand{\u}{ O{} O{} }{\textbf{u}_{#1}^{#2}}
\DeclareDocumentCommand{\v}{ O{} O{} }{\textbf{v}_{#1}^{#2}}
\DeclareDocumentCommand{\x}{ O{} O{} }{\textbf{x}_{#1}^{#2}}

\begin{document}

\begin{frontmatter}

\title{Realtime Safety Control for Bipedal Robots to \\Avoid Multiple Obstacles via CLF-CBF Constraints}

\author{
Jinze Liu$^*$,
Minzhe Li$^*$, Jessy W. Grizzle, and Jiunn-Kai Huang\fnref{fn1}}

\cortext[cor1]{equal contribution.}
\fntext[fn1]{Jinze Liu, Minzhe Li, Jessy W. Grizzle, and Jiunn-Kai Huang are with the Robotics
Department, University of Michigan, Ann Arbor, MI 48109, USA. \texttt{\{
jzliu, minzlee, grizzle, bjhuang\}@umich.edu}.
}

\begin{abstract} 
This paper presents a reactive planning system that allows a Cassie-series bipedal robot to avoid multiple non-overlapping obstacles via a single, continuously differentiable control barrier function (CBF). The overall system detects an individual obstacle via a height map derived from a LiDAR point cloud and computes an elliptical outer approximation, which is then turned into a CBF. The QP-CLF-CBF formalism developed by Ames \textit{et al.} is applied to ensure that safe trajectories are generated. Liveness is ensured by an analysis of induced equilibrium points that are distinct from the goal state. Safe planning in environments with multiple obstacles is demonstrated both in simulation and experimentally on the Cassie biped.  
\end{abstract}


\end{frontmatter}


%




\section{Introduction and Contributions}
Bipedal robots are typically conceived to achieve agile-legged locomotion over irregular terrains, and maneuver in cluttered environments\cite{huang2021efficient,gong2020angular, GrantMPC}. To explore safely in such environments, it is critical for robots to generate quick, yet smooth responses to any changes in the obstacles, map, and environment. In this paper, we propose a means to design and compose control barrier functions (CBFs) for multiple non-overlapping obstacles and evaluate the system on a 20-degree-of-freedom (DoF) bipedal robot. 


\begin{figure}[h!]%
\centering
\subfloat{%
\includegraphics[width=0.67\columnwidth]{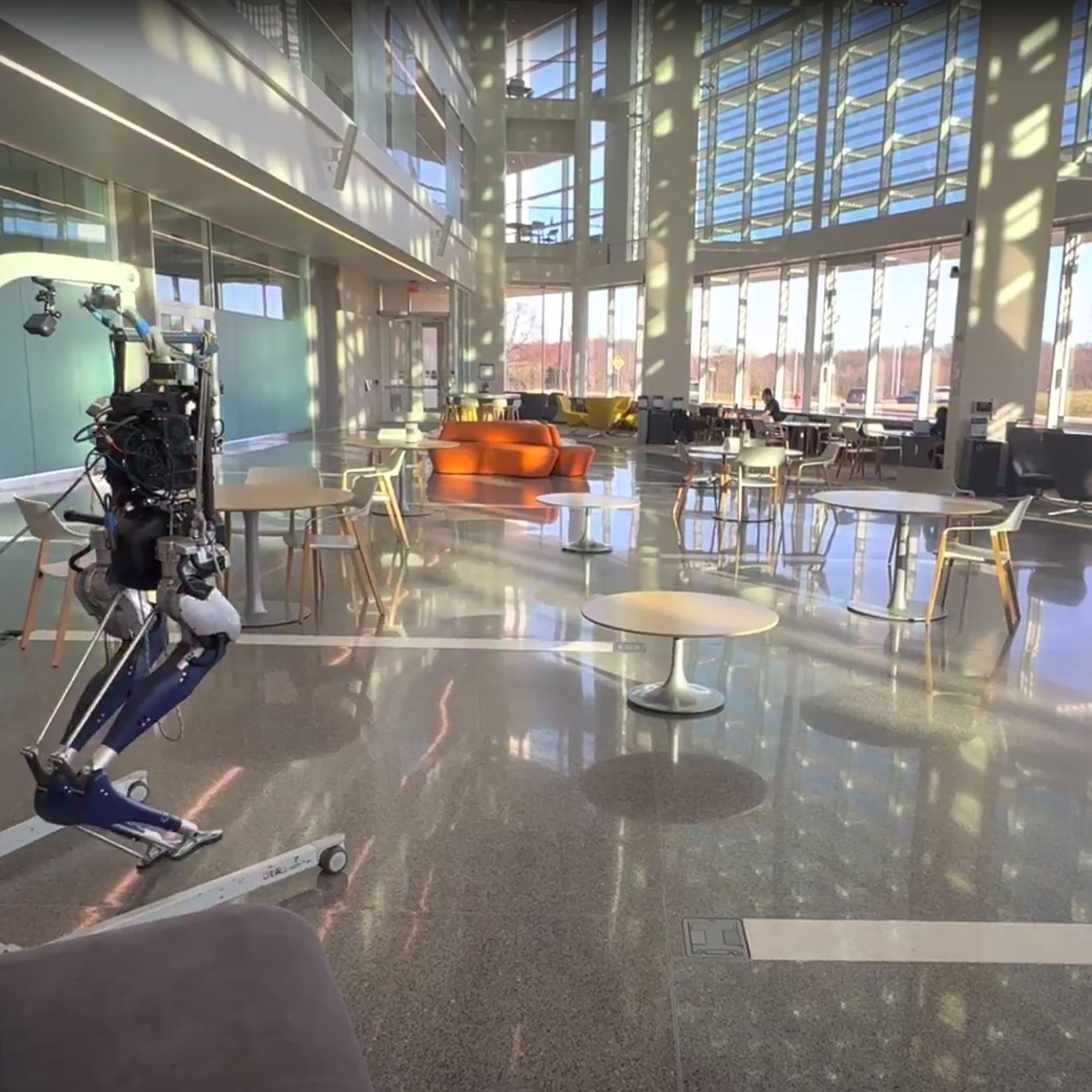}}\vspace{2pt}
\subfloat{%
\includegraphics[width=0.67\columnwidth]{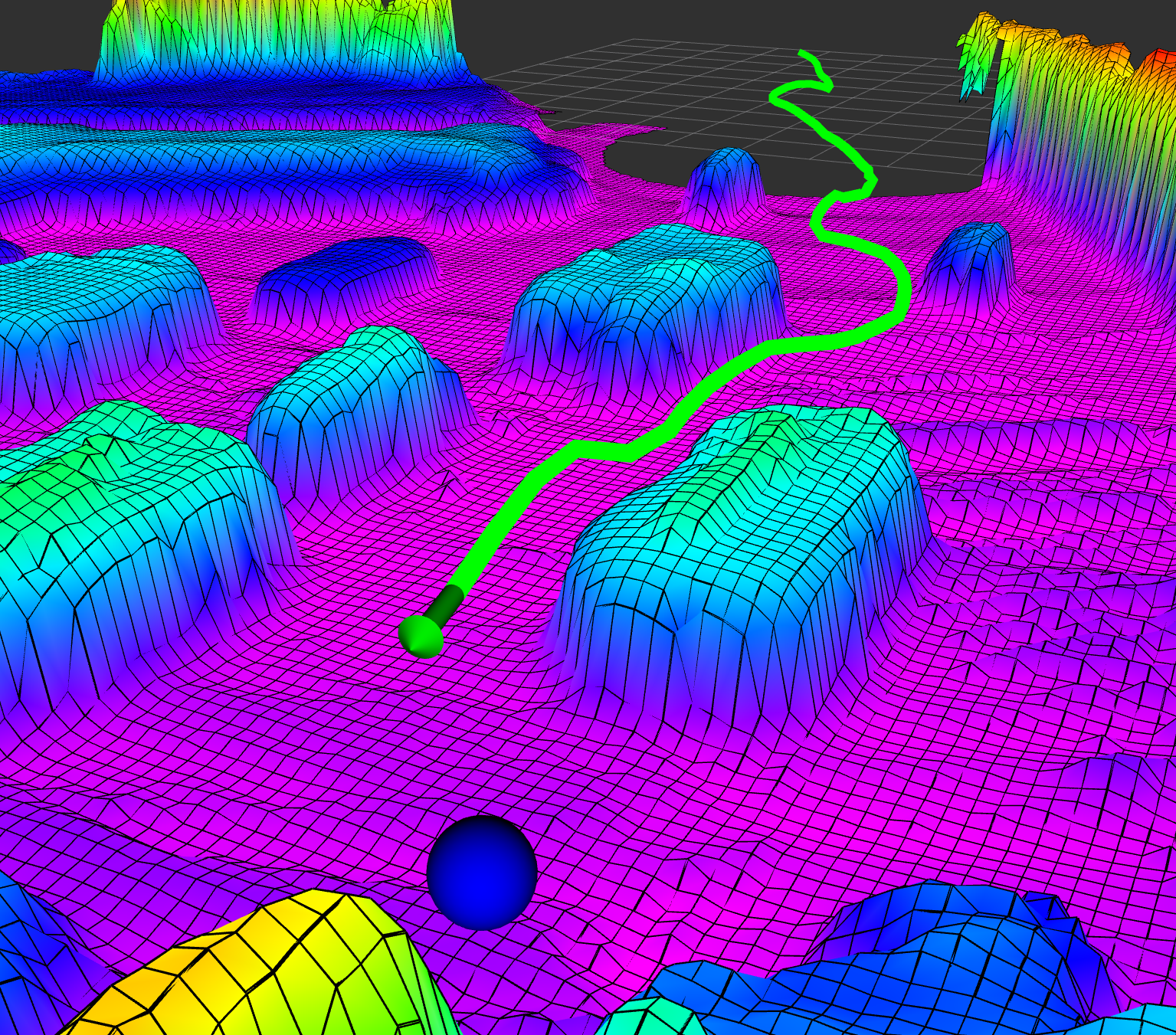}}
\caption[]{In the top figure, Cassie Blue autonomously avoids multiple obstacles via the developed CLF-CBF-QP obstacle avoidance system, comprised of an intermediate goal selector, obstacle detection, and a CLF-CBF quadratic programming solver. The bottom figure is the elevation map built in real time. The blue and cyan blobs are obstacles that Cassie detects and avoids in real time. A gantry is used in the experiments because the lab-built perception system that has been added to the robot is unprotected in case of falls.}
\label{fig:FirstImage}%
\end{figure}

In an autonomous system, the task of avoiding obstacles is usually handled by a planning algorithm because it has access to the map of an entire environment. Given the map, the planning algorithm is then able to design a collision-free path from the robot's current position to a goal. If the map is updated due to a change in the environment, the planner then needs to update the planned path, so-called replanning, to accommodate the new environment. Such maps are typically large and contain rich information such as semantics, terrain characteristics, and uncertainty, and thus are slow to update. This raises a concern when obstacles either move into the planned path but the map has not been updated or a robot's new pose allows the detection of previously unseen obstacles. The slow update rate of the map leads to either collision or abrupt maneuvers to avoid collisions. The non-smooth aspects arising from the map updates or changes in the perceived environment can be detrimental to the stability of the overall system. 

Research on obstacle avoidance has been studied for several decades as pioneered in classic probabilistic roadmap approaches (PRM) \cite{kavraki1996probabilistic} and cell decomposition \cite[Chapter 6]{lavalle2006planning}. However, the omission of robot dynamics and the extra computation for map discretization make these methods hard to use in real-time applications. Artificial potential fields \cite{borenstein1989real,hwang1992potential,valavanis2000mobile, montiel2015path, IntroLR:PotentialField2, IntroLR:PotentialField3, IntroLR:PotentialField4, IntroLR:PF5localminima, IntroLR:PF6oscillation, IntroLR:PF7oscillation2} stand out for their simplicity, extendability, and efficiency, leading to their wide adoption for real-time obstacle avoidance planning problems. A drawback of potential field methods is that they require the entire map of an environment to be available when designing a potential function that will render attractive one or more goal points in the map. Moreover, unwanted local minima and oscillations in the potential field have limited their deployment in the field. A control barrier function (CBF) \cite{LR:CBFCLFInsp}, on the other hand, enables real-time controller synthesis with provable safety for mobile robots operating in a continuous (non-discretized) space and can work with a partial (or incomplete) map. 

Control Lyapunov functions (CLFs) are positive definite functions such that at any given time instance, there exists a control input that renders the derivative of the function along the system dynamics negative definite. A CLF is typically designed to vanish at a desired goal state or pose.

The main theme of \cite{LR:CBFCLFInsp} is that a real-time quadratic program (QP) can be used to combine a CLF and a CBF in such a way that closed-loop trajectories induced by the CLF are minimally modified to provide provable safety, that is, non-collision with obstacles. This design philosophy has been explored in \cite{LR:CBFCLFInsp,LR:CLFCBF1, LR:CLFCBF2}. 

One means of avoiding obstacles is to come to a complete stop, though it is at the cost of not reaching the goal state. The papers \cite{reis2020control, LR:EquilibriaMultiAgent, LR:EquilibriaTan, LR:EquilibriaBall} showed that such behavior can be an unintended outcome of the CLF-CBF-QP design approach of \cite{LR:CBFCLFInsp}. Specifically, the inequality constraints (of the QP) associated with the derivatives of the control Lyapunov and control barrier functions can induce equilibria in the closed-loop system that are distinct from the equilibrium of the CLF. Reference \cite{reis2020control} characterizes these equilibria via the KKT conditions associated with the QP, while reference \cite{LR:EquilibriaMultiAgent} emphasizes that if an induced equilibrium is unstable, then ``natural noise'' in the environment will avoid the robot being deadlocked at an unstable equilibrium.


Inspired by the above-cited works on CLF-CBF-QPs for planning and control, we  incorporate high-bandwidth obstacle avoidance into the CLF-RRT* reactive planner of \cite{huang2021efficient}. The CLF in \cite{huang2021efficient} takes into account features specific to bipeds, such as the limited lateral leg motion that renders lateral walking more laborious than sagittal plane walking. This paper seeks to utilize the CLF designed specifically for bipedal robots in tandem with a CBF to avoid multiple, non-overlapping obstacles in a smooth fashion, while ensuring progress to a goal state.

The main contributions of the new proposed CLF-CBF system are the following:
\begin{enumerate}
  \item We propose a novel CLF-CBF-QP obstacle avoidance system specifically adapted for bipedal robots locomoting in the presence of multiple non-overlapping obstacles. The full system provides for real-time obstacle detection, CBF design, and safe control input generation through a QP. 
  \item We mathematically prove the validity of the proposed CBF for both single and multiple obstacles. We also analytically analyze the existence of spurious equilibrium points induced by the CLF-CBF constraints on the QP. 
  \item We propose a simple means to smooth and interpolate the discontinuous reference control variables, which are introduced when switching from one target to another. 
  \item We provide simulations that support the mathematical analysis for obstacle avoidance while reaching a goal.
  \item The overall reactive planning system is demonstrated experimentally on a  20-degree-of-freedom Cassie-series bipedal robot. 
  \item We open-source the implementations of the entire CLF-CBF system in C++ with Robot Operating System (ROS) \cite{ros} and associated videos of the experiments; see \href{https://github.com/UMich-BipedLab/multi_object_avoidance_via_clf_cbf}{https://github.com/UMich-BipedLab}\href{https://github.com/UMich-BipedLab/multi_object_avoidance_via_clf_cbf} {/multi\_object\_avoidance\_via\_clf\_cbf}.
\end{enumerate}

The rest of the paper is organized as follows. Section \ref{sec:RelatedWork} overviews related work. The design and validation of the proposed CBF is presented in Sec.~\ref{sec:Methods}. We analyze equilibrium points of the proposed CBF in \ref{sec:EquilibriumAnalysis}. Section \ref{sec:Multi-CBF} proposes a novel and simple method to combine CBFs for non-overlapping obstacles. Section \ref{sec:SmoothRefControl} introduces a method to smooth the reference control variables while switching target positions. Simulation and experimental results are given in Sec.~\ref{sec:Simulation} and Sec.~\ref{sec:Experiments}, respectively. Section \ref{sec:Conclusion} concludes the paper and provides potential future work.

\section{Related Work on Control with Safety}
\label{sec:RelatedWork}


A continuously differentiable, proper, positive definite function $V(x)$ that vanishes at a single point is called a candidate Lyapunov function \cite{khalil2002}. If the derivative of $V(x)$ along the trajectories of a control system can be rendered negative definite by proper choice of the control input, it is called a \textit{control Lyapunov Function}, or CLF for short  \cite{Sontag:firstCLF,Sontag:universal,FK:Book}. CLFs are widely used in the design of controllers to asymptotically drive a system to a goal state. Safety involves steering a control system to a goal state while avoiding self-collisions, obstacles, or other undesirable states, collectively referred to as \textit{unsafe states}. The set complement of the unsafe states is the set of \textit{safe states}.

\subsection{Artificial Potential Fields and Navigation Functions}

The first systematic method for real-time control and obstacle avoidance was introduced by Khatib in \cite{khatib1985real}. Called the method of \textit{Artificial Potential Functions}, it revolutionized feedback control for manipulators in that hard constraints could be enforced in both the robot's task space and joint space in real time. Prior to this seminal work, obstacle avoidance, or more generally the generation of safe paths, was relegated to a path planner operating at a much slower time scale. A survey of the method of potential functions can be found in \cite{koditschek1989robot}. 

Potential functions seek to construct ``repulsive fields'' around obstacles that are active throughout the entire state space of the robot's dynamical system, without destroying the presence of an attractive field steering the system to a goal state. It has been recognized that superimposed attracting and repelling fields can create undesired spurious equilibria, which prevent a robot from reaching its goal state \cite{koren1991potential}. In addition, potential fields have been observed to introduce trajectory oscillations as a robot passes near obstacles. Heuristic modifications have been proposed to avoid local minima \cite{IntroLR:PotentialField3,IntroLR:PotentialField4,IntroLR:PF5localminima}, while potential fields have been combined with other gradient-based functions to reduce oscillations \cite{IntroLR:PF6oscillation,IntroLR:PF7oscillation2}. 



The method of \textit{Navigation Functions} by Koditschek and Rimon \cite{rimon1990exact} sought to design a single function whose gradient produces trajectories avoiding multiple obstacles while asymptotically converging to a single goal state from \textit{almost all initial conditions} \cite{koditschek1990robot, ExactRobot2016, arslan2019sensor, paternain2017navigation}; specifically, all equilibria except the goal state should be unstable. Because the design of a navigation function takes into account the global topology of the method of navigation functions is not appropriate for problems requiring the online identification and avoidance of obstacles; in addition, there are topological restrictions to the existence of navigation functions.

\subsection{Control Barrier Functions and Control Lyapunov Functions}

\textit{Barrier Functions} provide Lyapunov-like conditions for proving a given set of safe states is forward invariant, meaning that trajectories starting in the safe set remain in the safe set. The natural extension of a barrier function to a system with control inputs is a \textit{Control Barrier Function} or CBF for short, first proposed by \cite{wieland2007constructive}. CBFs parallel the extension of Lyapunov functions to CLFs, in that the key point is to impose inequality constraints on the derivative of a candidate CBF (resp., CLF) to establish entire classes of controllers that render a given set forward invariant (resp., asymptotically stable).

Importantly, barrier functions and CBFs focus solely on safety and do not seek to simultaneously steer a system to any particular point in the safe set. This allows CBFs to be combined with other ``goal-oriented'' control methods as a (maximally permissive) supervisor that only modifies a trajectory when it is in conflict with the safety criteria established by the CBF. The papers \cite{LR:CBFvsRA, LR:CBFSafety-Critical} introduced the notion of using a real-time quadratic program (QP) to combine a CBF with a CLF to achieve convergence to a goal state while avoiding unsafe states. The overall method goes by the acronym CLF-CBF-QP. 

For control systems that are affine in the control variable, CLF-CBF-QPs have proven to be enormously popular in and out of robotics applications \cite{LR:CBFCLFInsp, LR:CLFCBF1, LR:CLFCBF2, LR:CLFCBF3, LR:CLFCBF4, desai2022clf,basso2020safety, agrawal2017discrete}.
To highlight just a few example, reference \cite{LR:CLFCBF1} uses a CLF-CBF-QP to achieve stable walking for bipedal robots, while trajectory planning under spatiotemporal and control input constraints is presented in \cite{LR:CLFCBF4,LR:CLFCBF2, LR:CLFCBF3}. Applications to obstacle avoidance are addressed in \cite{agrawal2017discrete, desai2022clf, basso2020safety}. 

The recent paper \cite{IntroLR:PFvsCBF} shows that CBFs are a strict generalization of artificial potential functions and that in a practical example, a CLF-CBF-QP has reduced issues with oscillations as a robot passes near obstacles and improved \textit{liveness}, meaning the ability to reach the goal state. Hence, we use the method of CLF-CBF-QPs in this paper.

\subsection{Combining Multiple CBFs}

Usually, a control barrier function is designed for a single obstacle. When there are multiple obstacles in the control system, the barrier functions for each obstacle must be combined in some manner to provide safety guarantees. Reference \cite{rauscher2016constrained} 
shows that if the intersection of the set of ``allowable controls'' of individual CBFs is non-empty, then the CLF-CBF-QP method can be extended to several obstacles; the reference does not show how to check this condition online (in real time). Multiple CBF functions have also been combined to obtain a single CBF so that existing methods can be applied. Reference \cite{glotfelter2017nonsmooth} combines several CBFs into an overall CBF using max-min operations. The resulting CBF is non-differentiable and hence this technique is not used here. 
Reference \cite{romdlony2016stabilization} combines multiple CBFs for disjoint unsafe sets with a single CLF to produce a new CLF that simultaneously provides asymptotic stability and obstacle avoidance. This work is therefore related to the method navigation functions reviewed above and suffers from the same drawbacks; however, a key technique used in this reference to combine the CBFs before merging them with a CLF will be exploited in the current paper, namely a continuously differentiable saturation function.

\subsection{CLF-CBF-QPs and Unwanted Equilibrium Points}
The presence of multiple stable equilibrium points introduces ``deadlock'' in a control system.
Reference \cite{reis2020control} shows that the use of real-time QPs to combine safety and goal-reaching in navigation problems can lead to unwanted equilibrium points. With this awareness, the authors of \cite{LR:EquilibriaTan} modify the cost function in the quadratic program to remove the unwanted equilibria. The modification induces a rotational motion in the closed-loop system that steers it around the obstacle, something a bipedal robot can do naturally. Hence, here we only exploit their analysis method for finding the unwanted equilibria and show that our method introduces at most one undesired equilibrium point when obstacles are disjoint. Moreover, we do not need to remove the unwanted equilibrium using the methods in \cite{LR:EquilibriaCollisionCone, LR:EquilibriaBall} by transforming the system's state space into a convex manifold, or by increasing the complexity of the system's state  space.

 \subsection{Summary}
The presence of multiple obstacles is common in practice. While existing works can treat disjoint obstacles, they are not appropriate for use where obstacles are identified in real-time via an onboard perception system. In this work, for a biped-appropriate planning model, we propose a simple means to combine CBFs for disjoint obstacles so that the complexity of the real-time CLF-CBF-QP remains constant and induced equilibrium points are easy to characterize and avoid.

\section{Construction of Control Lyapunov Function and Control Barrier Function}
\label{sec:Methods}


This section introduces the CLF proposed in \cite{huang2021efficient} and analyzes its trajectories when combined with a quadratic CBF through a real-time QP. The goal is to ensure the closed-loop system is able to reach a goal state while smoothly avoiding a single obstacle. This section lays the foundation for considering multiple obstacles in the next section.

\subsection{State Representation}
\label{sec:StateRepresentation}
Denote $\Pcal=(x_r,y_r,\theta)$ the robot pose, $\Gcal=(x_t,y_t)$ the goal position in the world frame. We simplify an obstacle $\Ocal$ as a circle (and hence convex) described as its center $(x_o, y_o)$ and its radius $r_o$. We define the robot state as 
\begin{equation}
\label{eq:RobotState}
    x = \begin{bmatrix}
            r \\ 
            \delta\\
            \theta
        \end{bmatrix}, 
\end{equation}
where $r = \sqrt{(x_t-x_r)^2+(y_t-y_r)^2}$, $\theta$ is the heading angle of the robot, and $\delta$ is the angle between $\theta$ and the line of sight from the robot to the goal, as shown in Fig.~\ref{fig:PhysicalMeaning}. 

The dynamics of the control system is defined as
\begin{equation}
\label{eq:ControlSystem}
\begin{split}
    \dot{x} &= f(x) + g(x)u \\
    &= 
    \begin{bmatrix}
        0 \\ 
        0 \\
        0
    \end{bmatrix}
    +
    \begin{bmatrix}
        -\cos(\delta) & -\sin(\delta) & 0 \\ 
        \frac{\sin(\delta)}{r} & -\frac{\cos(\delta)}{r} & 1 \\
        0 & 0 & -1
    \end{bmatrix}
    \begin{bmatrix}
        v_x \\ 
        v_y \\
        \omega
    \end{bmatrix},
\end{split}
\end{equation}
where we view $u=\begin{bmatrix} v_x, &v_y, &\omega\end{bmatrix}^\transpose$ as the control variables in the robot frame, as shown in Fig.~\ref{fig:PhysicalMeaning}.

\begin{figure}[t]
    \centering
    \includegraphics[width = 0.5\textwidth]{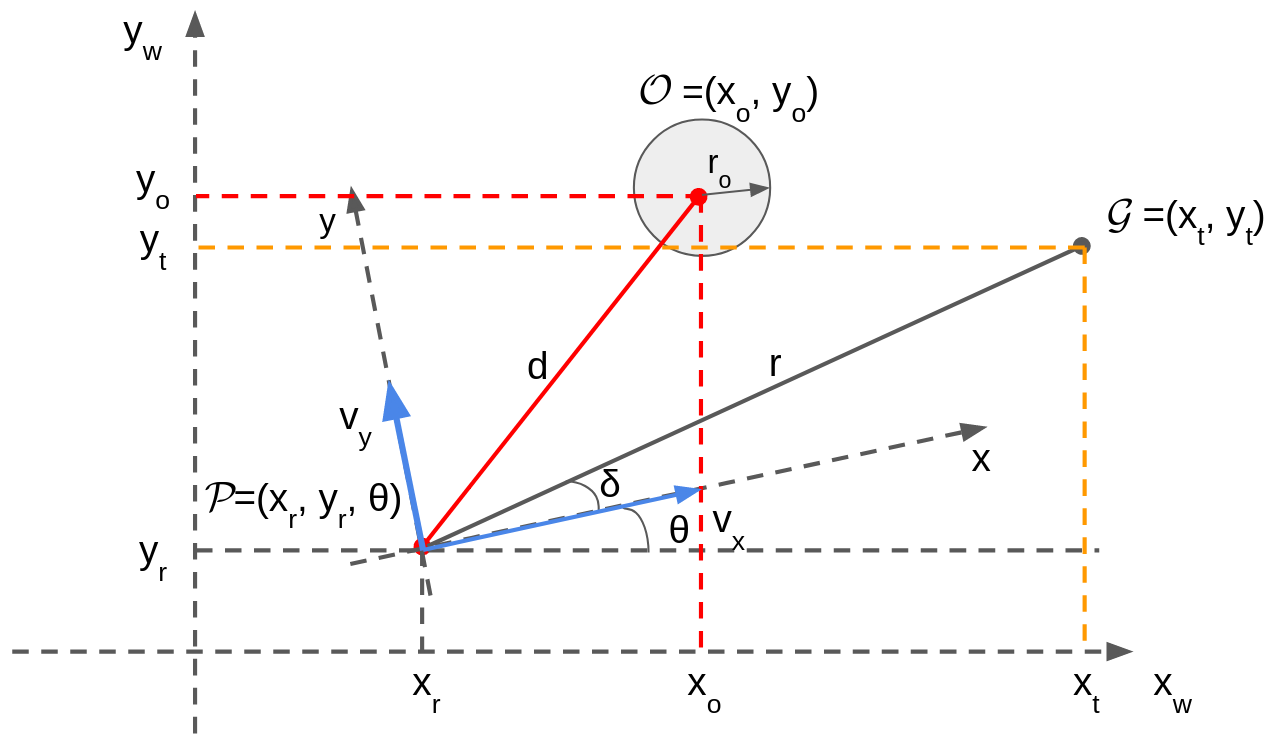}
    \caption{Illustration of the robot state representation. $x_w$ and $y_w$ are the axes of the world frame. The robot's pose and the target position are $\Pcal = (x_r, y_r, \theta)$ and $\Gcal = (x_t, y_t)$ in the world frame, respectively. $x$ and $y$ are the axes of the robot frame where the positive $x$ direction is the robot heading direction. $r$ is the distance between the robot and the target position and $\delta$ is the angle between the robot heading and the target direction. $(x_o, y_o)$ is the center of obstacle in the world frame and $r_o$ is the radius of the obstacle. $d$ is the distance between the obstacle and the robot.}
    \label{fig:PhysicalMeaning}
\end{figure}

\subsection{Design of CLF and CBF for Bipedal Robots}
\label{sec:CLF&CBF}
The control Lyapunov function leveraged in the reactive planner proposed in \cite{huang2021efficient} takes into account features specific to bipeds, such as the limited lateral leg motion that renders lateral
walking more laborious than sagittal plane walking. Therefore, we also define the CLF as
\begin{equation}
\label{eq:CLF}
    V(x) = \frac{r^2 + \gamma^2\sin^2(\beta\delta)}{2},
\end{equation}
where $\gamma$ is the weight on the orientation, and $\beta$ controls the size of the field of view (FoV). Given $\Pcal$ and $\Gcal$, we have a closed-form solution for control $u$ in \eqref{eq:ControlSystem}, 
\begin{equation}
\label{eq:CLFSolution}
\begin{aligned}
    \omega^\mathrm{ref} &= \frac{r\cos(\delta) \left[ rv_\delta\cos(\delta)-v_r\sin(\delta) \right]}{\alpha + r^2 \cos^2(\delta) }\\
    v_y^\mathrm{ref} &= \frac{\alpha(v_r\sin(\delta)-rv_\delta\cos(\delta))}{r^2{\cos(\delta)}^2+\alpha}\\
    v_x^\mathrm{ref} &= \frac{v_r\cos(\delta)r^2+\alpha v_\delta\sin(\delta)r+\alpha v_r\cos(\delta)}{r^2{\cos(\delta)}^2+\alpha};
\end{aligned}
\end{equation}
where $v_r$ and $v_\delta$ are defined as:
\begin{equation}
\label{eq:feedback}
\begin{aligned}
v_r &= k_{r1}\frac{r}{k_{r2}+r}\\
v_\delta &= -\frac{2}{\beta}k_{\delta 1}\frac{r}{k_{\delta 2}+r}\sin(2\beta\delta).
\end{aligned}
\end{equation}
In \eqref{eq:CLFSolution} and \eqref{eq:feedback}, $\alpha, \beta, k_{r1}, k_{r2}, k_{\delta1}, k_{\delta2}$ are positive constants. See \cite{huang2021efficient} for more details.

Next, we introduce a candidate CBF as
\begin{equation}
\label{eq:OriginalCBF}
\begin{aligned}
    B(x) = \left[\begin{array}{c} x_{r} - x_o \\ y_{r} - y_o  \end{array}\right]^\top Q  \left[\begin{array}{c} x_{r} - x_o \\ y_{r} - y_o  \end{array} \right] - r_o^2,
\end{aligned}
\end{equation}
where $(x_o, y_o)$ gives the center of the obstacle, $r_o$ specifies the ``radius'' of the obstacle, and $Q$ is positive definite. 
We next verify that \eqref{eq:OriginalCBF} is a valid CBF. 


\subsection{Proof of CBF Validity}
\label{subsec:ProofOfCBF}

Following \cite{ames2019control}, we define the sets 
\begin{equation}
    \label{eq:hGreaterThanZero}
    \begin{aligned}
    {\cal D} &:= \{ x \in \real^3~|~ B(x) \neq -r_o^2, \text{ and } r \neq 0\} \\
    {\cal C} & :=\{ x \in {\cal D} ~|~ B(x) \ge 0 \}
    \end{aligned}    
\end{equation}
associated with the candidate CBF \eqref{eq:OriginalCBF} and note that ${\rm Int}(\cal C) \neq \emptyset$ and $\overline{{\rm Int}(\cal C) }= {\cal C}$. From \cite{ames2019control}, for \eqref{eq:OriginalCBF} to be a valid CBF function of \eqref{eq:ControlSystem}, there must exist some $\eta > 0$, such that,
\begin{equation}
\label{eq:CBFrequirement}
\forall x \in {\cal D}, {\exists u \in 
\real^3}, \dot{B}(x,u) +\eta B(x) \geq 0,
\end{equation}
where $\dot{B}(x,u):=L_fB(x) + L_gB(x)u$ is the time derivative of $B(x)$ along the dynamics of  \eqref{eq:ControlSystem}, $\eta>0$ sets the repulsive effort of the CBF, and
\begin{align}
    \label{eq:LieDerivatives}
    L_fB(x) &:= \frac{\partial B(x) }{\partial x} f(x) \\
    L_gB(x) &:= \frac{\partial B(x) }{\partial x} g(x).
\end{align}

Because the drift term $f(x)$ in \eqref{eq:ControlSystem} is identically zero, the zero control $u\equiv 0$ satisfies \eqref{eq:CBFrequirement} for $x \in {\cal C}$. Hence, we need to show that \eqref{eq:CBFrequirement} can be met for $x \in  \thicksim {\cal C}$, the set complement of ${\cal C}$. Direct application of the chain rule gives that $$L_g B(x) = a(x) b(x) g(x),$$ where
\begin{equation}
\label{eq:LgB}
\begin{aligned}
a(x) & :=  \small 2\left[\begin{array}{ll} x_{t} - r\cos(\delta+\theta) - x_o, & y_{t} - r\sin(\delta+\theta) - y_o\end{array}\right]Q \medskip \\
& ~= 2 \left[ \begin{array}{ll}x_r - x_o, &y_r - y_o\end{array} \right] Q \\
    b(x)&:=     \left[ \begin{array}{ccr}
                     -\cos(\delta+\theta) & r\sin(\delta+\theta) & r\sin(\delta+\theta) \\ 
                     -\sin(\delta+\theta) & -r\cos(\delta+\theta) & -r\cos(\delta+\theta) \end{array} \right] \medskip \\ 
 g(x) &~= \left[ \begin{array}{ccr}
                     -\cos(\delta) & -\sin(\delta) & 0 \medskip \\ 
                     \frac{\sin(\delta)}{r} & -\frac{\cos(\delta)}{r} & 1  \medskip\\
                     0 & 0 & -1
                \end{array} \right].
\end{aligned}
\end{equation}
Moreover, $a(x)$ only vanishes at the center of an obstacle, the rows of $b(x)$ are linearly independent for all $r>0$, and 
${\rm det} \left( g(x) \right) = -\frac{1}{r} \neq 0$ for all $0 < r < \infty$. It follows that for all $x \in {\cal D}$, $L_g B(x) \neq 0$ and hence \eqref{eq:CBFrequirement} is satisfied, proving that \eqref{eq:OriginalCBF} is a valid CBF.

\begin{figure}[t]
    \centering
    \includegraphics[width = 0.5\textwidth]{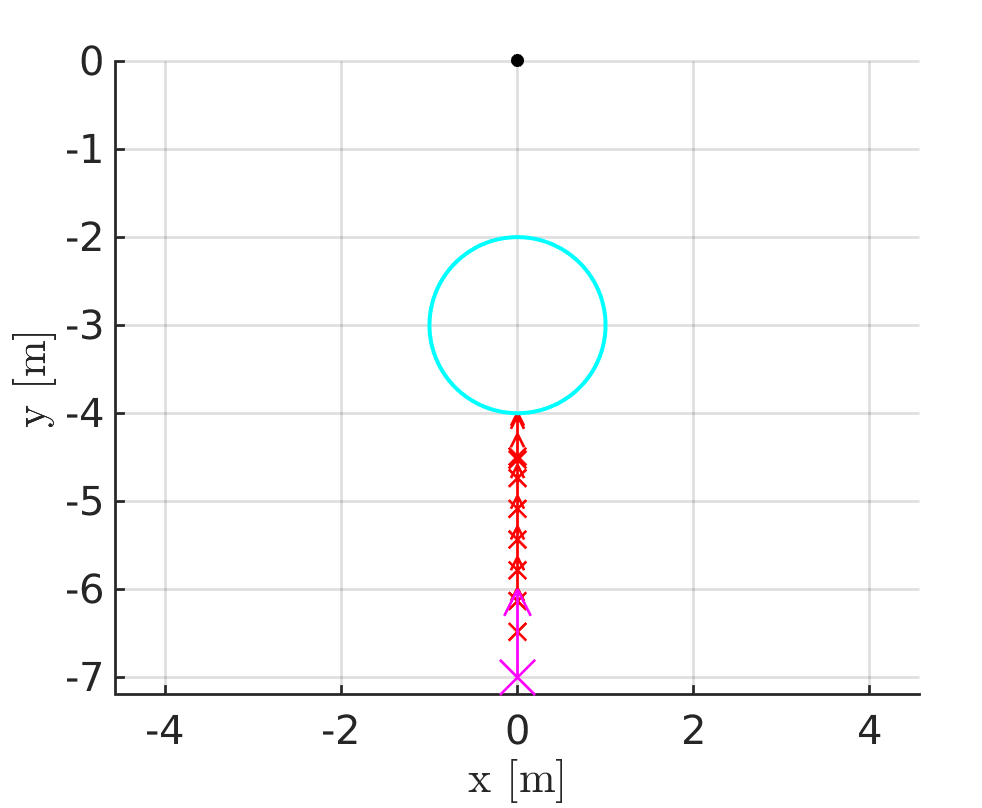}
    \caption{Illustration of a case when the robot directly faces the obstacle and the target creates an equilibrium in the continuous-time system. In a simulation with discrete-time control updates, the robot walks back and forth at the obstacle boundary.}
    \label{fig:Equilibrium1}
\end{figure}

\subsection{Quadratic Program of the Proposed CLF-CBF System}
\label{subsec:QP-of-cbf-clf}
A quadratic program (QP) is set up to optimize the control $u$ with the slack variable $s$ while enforcing both the CLF and CBF constraints. Let $\mathfrak{L}(x,u,s)$ be the CLF constraints
\begin{equation}
\label{eq:QP-clf-constraint}
    \mathfrak{L}(x,u,s) := L_fV(x)+L_gV(x)u+\mu  V(x)-s \le 0,
\end{equation}
where $L_pq(x):=\nabla q(x) \cdot p(x)$ is the Lie derivative, $\mu$ serves as a decay rate of the upper bound of $V(x)$. Next, we denote $\mathfrak{B}(x,u)$ the CBF constraints
\begin{equation}
\label{eq:QP-cbf-constraint}
    \mathfrak{B}(x,u) :=-L_fB(x)-L_gB(x)u-\eta B(x) \le 0,
\end{equation}
where $\eta$ serves as a decay rate of the lower bound of $B(x)$.

Finally, the QP for the control values is formulated as
\begin{equation}
u^*, s^*=\argmin_{\substack{\mathfrak{L}(x,u,s)\leq 0 \\ \mathfrak{B}(x,u)\leq 0}}\ \mathfrak{J}(u,s),
\label{eq:CBLF-QP}
\end{equation}
where the cost function $\mathfrak{J}(u,s)$ is defined as
\begin{equation}
\label{eq:QP-cost}
    \mathfrak{J}(u,s):=\frac{1}{2}(u-u^\mathrm{ref})^TH(u-u^\mathrm{ref})+\frac{1}{2}ps^2,
\end{equation}
the positive definite, diagonal matrix $H := {\rm diag}([h_1, h_2, h_3])$ weights the control variables, $u^\mathrm{ref}:=\begin{bmatrix} v_{x}^\mathrm{ref} & v_{y}^\mathrm{ref} & \omega^\mathrm{ref} \end{bmatrix}^\transpose$ is the control vector from the CLF \eqref{eq:CLF} without obstacles, and $p\ge 0$ is the weight of the slack variable, $s$.

In the proposed CLF-CBF-QP system, $u^\mathrm{ref}$ is the closed-form solution obtained from the CLF without obstacles, and $H$ assigns weights for different control variables. The proposed CLF-CBF-QP cost function captures inherent features of a Cassie-series robot, such as the low-cost of longitudinal movement and high-cost of lateral movement, while guaranteeing safety. We next look at liveness, that is, the ability of the system to reach the desired goal.

\begin{figure}[t]
    \centering
    \includegraphics[width = 0.5\textwidth]{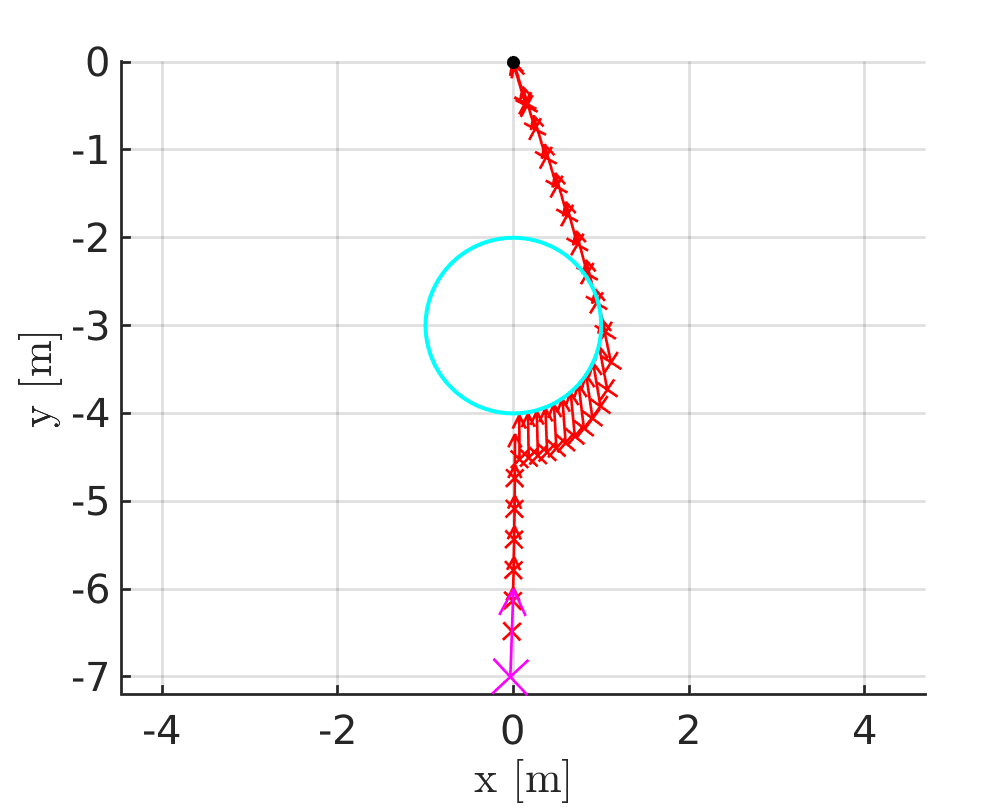}
    \caption{Illustration of breaking the equilibrium by using $u^\mathrm{ref}_2$: $\begin{bmatrix}
        v_{x}^\mathrm{ref} & v_{y}^\mathrm{ref} & \omega^\mathrm{ref} + \epsilon
    \end{bmatrix}^\transpose$ when $\delta=0$. The robot successfully reaches the target position without colliding with the obstacle.}
    \label{fig:BreakEquilibrium}
\end{figure}

\begin{figure*}[t!]%
    \centering
    \subfloat[]{%
    \label{fig:nonsmooth1}%
    \includegraphics[trim=0 0 0 0,clip,width=0.5\columnwidth]{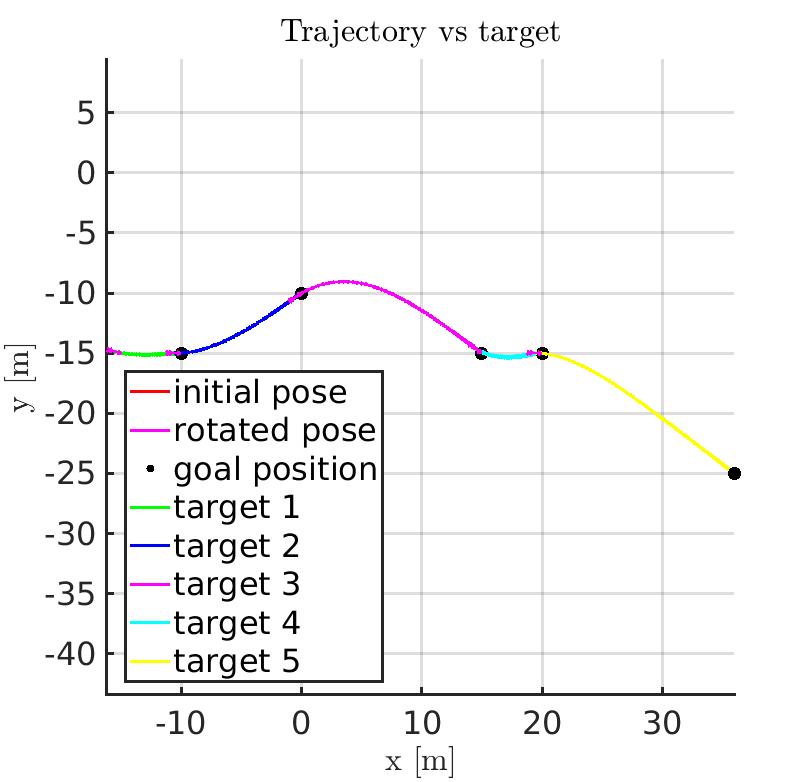}}~
    \subfloat[]{%
    \label{fig:nonsmooth2}%
    \includegraphics[trim=0 0 0 0,clip,width=0.5\columnwidth]{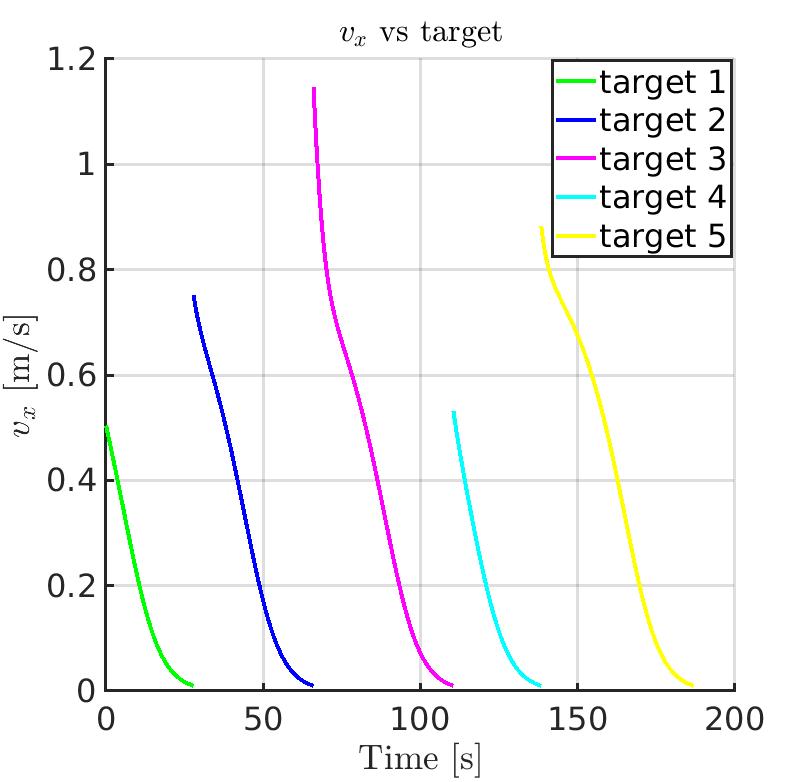}}~
    \subfloat[]{%
    \label{fig:nonsmooth3}%
    \includegraphics[trim=0 0 0 0,clip,width=0.5\columnwidth]{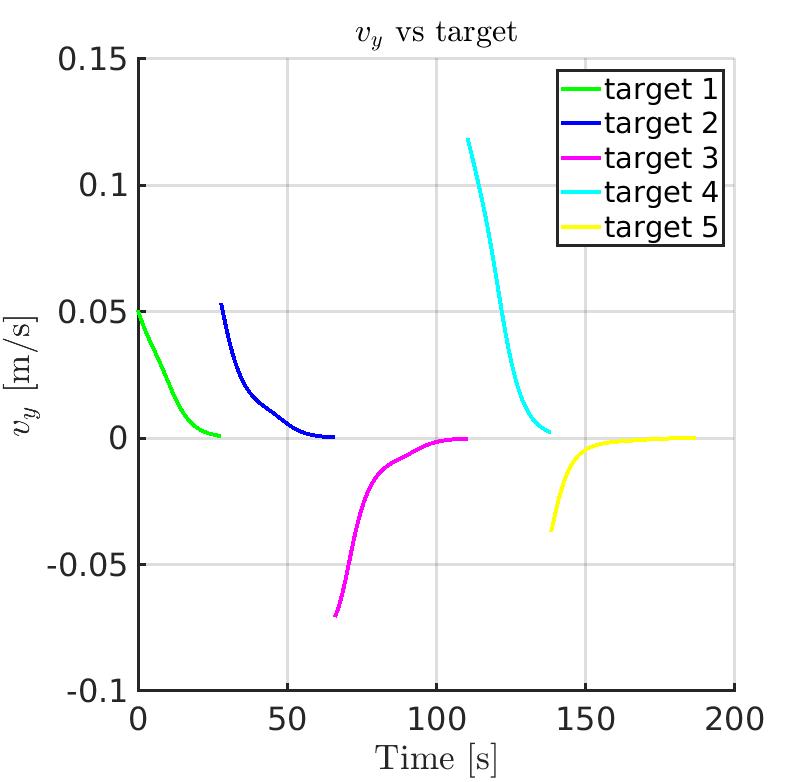}}~
    \subfloat[]{%
    \label{fig:nonsmooth4}%
    \includegraphics[trim=0 0 0 0,clip,width=0.5\columnwidth]{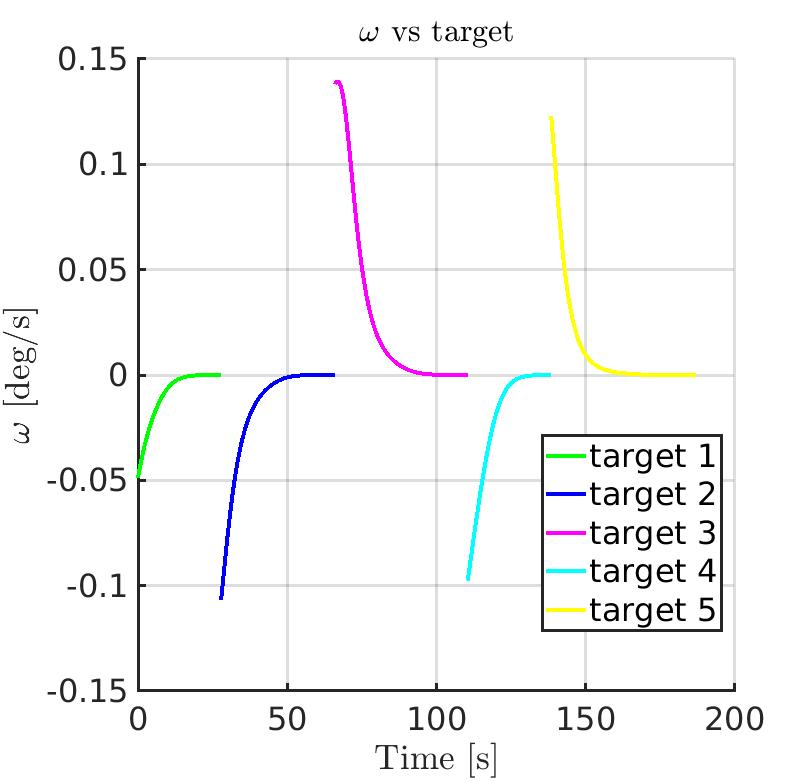}}//
    \subfloat[]{%
    \label{fig:smooth1}%
    \includegraphics[trim=0 0 0 0,clip,width=0.5\columnwidth]{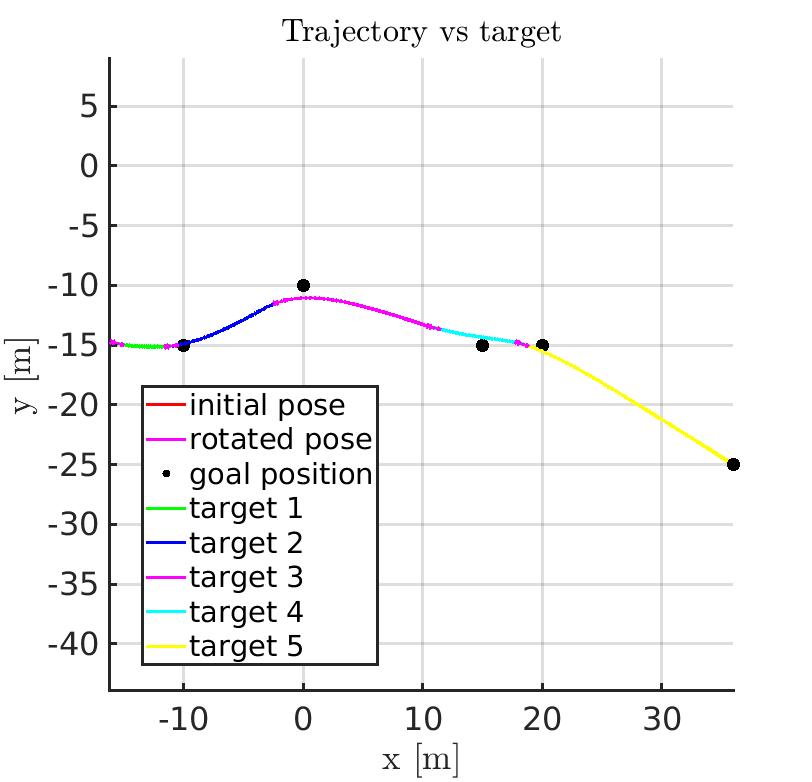}}~
    \subfloat[]{%
    \label{fig:smooth2}%
    \includegraphics[trim=0 0 0 0,clip,width=0.5\columnwidth]{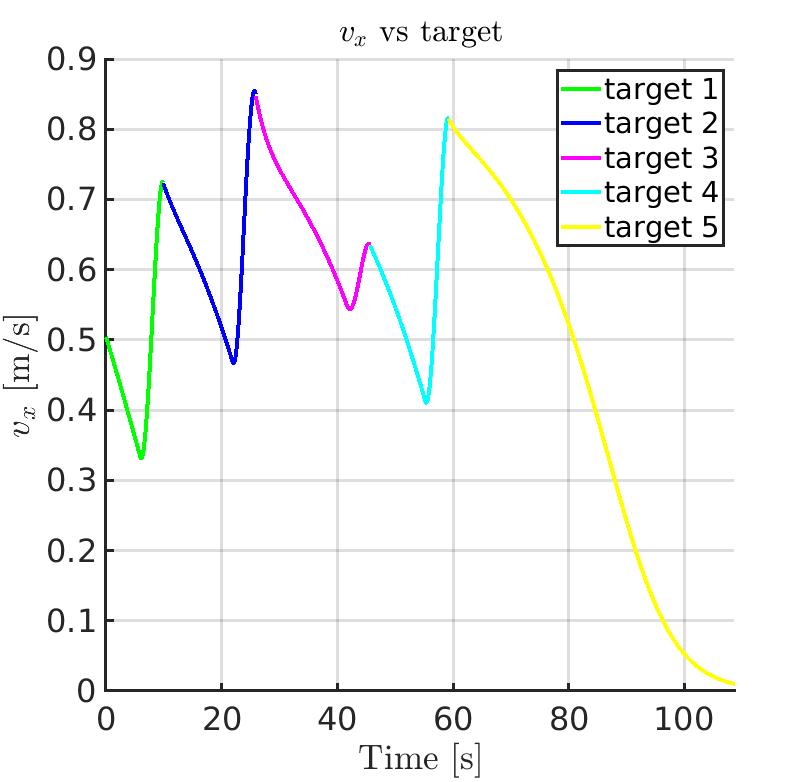}}~
    \subfloat[]{%
    \label{fig:smooth3}%
    \includegraphics[trim=0 0 0 0,clip,width=0.5\columnwidth]{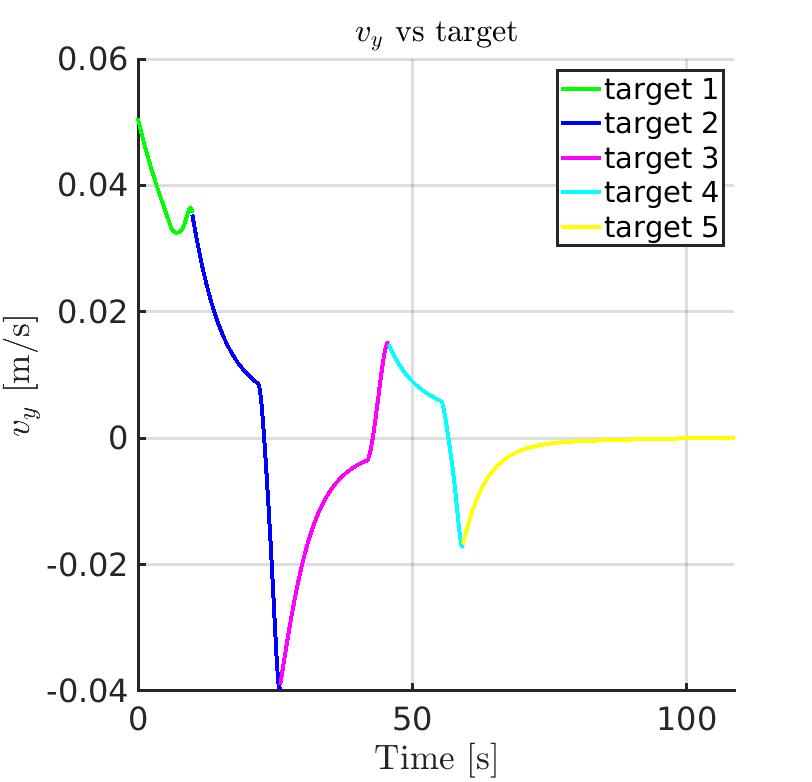}}~
    \subfloat[]{%
    \label{fig:smooth4}%
    \includegraphics[trim=0 0 0 0,clip,width=0.5\columnwidth]{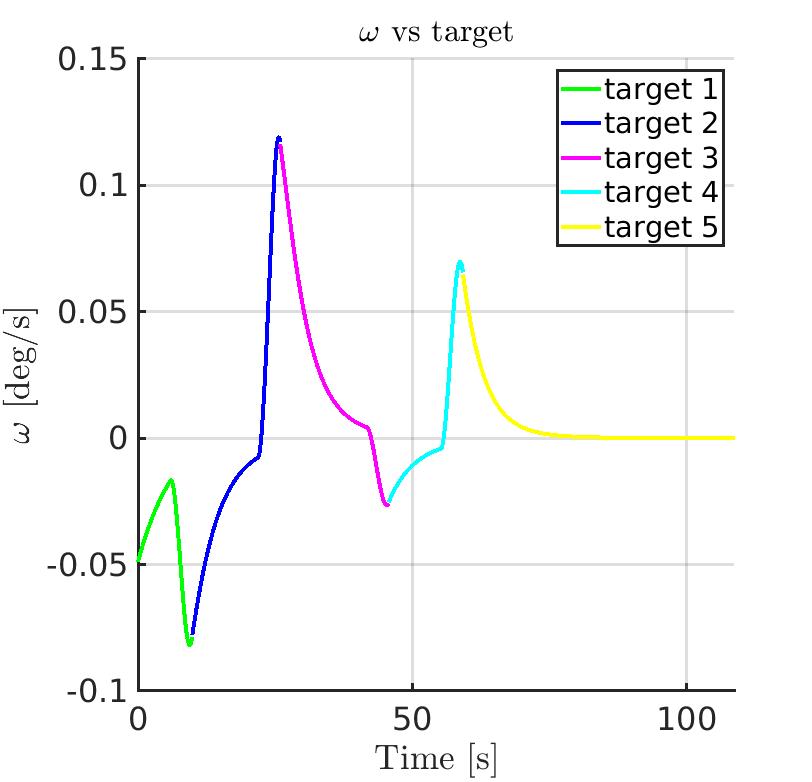}}
    \caption[]{The top and bottom rows are the resulting trajectories and control variables generated by \eqref{eq:CLFSolution} and \eqref{eq:smooth-uref}, respectively. The robot starts at $(-15, -15, -15^\circ)$, and the list of targets to reach is $(-10,-15)$, $(0,-10)$, $(15,-15)$, $(25,-15)$, and the final destination is $(40,-30)$.  Different colors represent the trajectory with different targets and also correspond to control variables for the trajectory. The control variables generated by \eqref{eq:CLFSolution} decrease to zero each goal is approached, which leads to discontinuous control variables.
    On the other hand, with the smoothing process, the control variables are continuous and much smoother, while the resulting trajectories are similar.}
    \label{fig:nonsmooth target switch}
\end{figure*}

\subsection{Analysis for Unwanted Equilibria}
Paper \cite{reis2020control} points out very clearly that the CLF-CBF-QP formulation of Sec.~\ref{subsec:QP-of-cbf-clf} can introduce unwanted equilibria that may prevent the robot from reaching a goal state. The paper \cite{LR:EquilibriaMultiAgent} also considered this problem and noted that if the equilibria are unstable, then liveness is preserved for almost all initial conditions. In \ref{sec:EquilibriumAnalysis}, we follow the KKT-analysis of the CLF-CBF-QP presented in \cite{reis2020control} and show that only one equilibrium point is created by the QP. Moreover, the equilibrium occurs at an obstacle boundary for $\delta=0,d_y=0,d_x>0$, in other words, when the robot's heading faces directly to the obstacle and the target, as shown in Fig.~\ref{fig:Equilibrium1}. The robot will move directly toward the obstacle and stop at the obstacle boundary.

\begin{remark}
When the robot encounters the above equilibrium state, we can add a constant $\epsilon>0$ to $u^\mathrm{ref}$ in \eqref{eq:CBLF-QP} such that $u^\mathrm{ref}=\begin{bmatrix}
        v_{x}^\mathrm{ref} & v_{y}^\mathrm{ref} & \omega^\mathrm{ref} + \epsilon
    \end{bmatrix}^\transpose$.
As is shown in Fig.~\ref{fig:BreakEquilibrium}, the robot breaks its equilibrium state, avoids the obstacle, and reaches the target position. This is related to, but distinct from, the method presented in \cite{reis2020control} for resolving unwanted equilibria.
\end{remark}





\section{Combining CBFs for Multiple Obstacles}
\label{sec:Multi-CBF}
So far, we have assumed there is only one obstacle perceived by the robot. In this section, we will discuss how to handle multiple obstacles in the environment when each obstacle is a positive distance apart from the others \cite{romdlony2016stabilization}. Specifically, for $i \in \{ 1, 2, \dots, M\}$, suppose that 
\begin{equation}
\label{eq:FamilyCBFs}
\begin{array}{l}
B_i(x):= \left[\begin{array}{c} x_{r} - x_{o,i} \\ y_{r} - y_{o,i}  \end{array}\right]^\top Q_i  \left[\begin{array}{c} x_{r} - x_{o,i}\\ y_{r} - y_{o,i}  \end{array} \right] - r_{o,i}^2\\
\\
{\cal D}_i := \{ x \in \real^3~|~ B_i(x) \neq -r_{o,i}^2, \text{ and } r \neq 0\} \\
\\
    {\cal C}_i  :=\{ x \in {\cal D}_i ~|~ B_i(x) \ge 0 \}
    \end{array}
\end{equation}
are valid CBF functions for the dynamics \eqref{eq:ControlSystem}. For $i \neq j$, the obstacles corresponding to 
 $B_i:\real^3 \to \real $ and $B_j:\real^3 \to \real$ are a \textbf{positive distance apart} if
\begin{equation}
    \label{eq:PosDistApart}
   \Delta_{ij}:= \underset{ \begin{array}{c} x \in \sim {\cal C}_i \\ y \in \sim {\cal C}_j\end{array} }{\inf } ||x - y|| >0.
\end{equation}

A key innovation with respect to \cite{glotfelter2017nonsmooth} is that we will compose the associated CBFs in a smooth ($C^1$) manner. A potential drawback with respect to \cite{glotfelter2017nonsmooth} is that we will assume the obstacles giving rise to the CBFs are a positive distance apart. Similar to \cite{romdlony2016stabilization}, we saturate standard quadratic CBFs before seeking to combine them. Distinct from \cite{romdlony2016stabilization}, we multiply the saturated CBFs instead of creating a weighted sum. This greatly simplifies the analysis of the composite CBF with respect to all previous works.

\subsection{Smooth Saturation Function}

We introduce a continuously differentiable saturation function that will allow us to compose in a simple manner CBFs corresponding to obstacles that are a positive distance apart. Consider $\sigma: \real \to \real$ by
\begin{equation}
    \label{eq:SmoothSaturationFunction}
    \sigma(s) : = \begin{cases}
        s & s \le 0 \\
        s(1 + s - s^2) & 0 < s < 1 \\
        1 & s \ge 1.
    \end{cases}
\end{equation}
Then straightforward calculations show that for all $s \in \real$, $\frac{d\sigma(s)}{ds}$ exists and satisfies
\begin{equation}
    \label{eq:deltaSmoothSaturationFunction}
   \frac{d\sigma(s)}{ds} : = \begin{cases}
        1 & s \le 0 \\
        1 +2 s - 3s^2 & 0 < s < 1 \\
        0 & s \ge 1.
    \end{cases}
\end{equation}
Upon noting that $\left.\frac{d\sigma(s)}{ds}\right|_{s=0}=1$, $\left.\frac{d\sigma(s)}{ds}\right|_{s=1}=0$ and for all $0 < s< 1$, $ 0 < \frac{d\sigma(s)}{ds}$, it follows that $\sigma: \real \to \real$ is continuously differentiable and monotonic.

\begin{remark} For $0 \le s \le 1 $, $\sigma$ is constructed from a degree-three B\'ezier polynomial $p:[0,1] \to \real$ such that $p(0)=0$, $\frac{dp(0)}{ds}=1$, $p(1)=1$, $\frac{dp(1)}{ds}=0$. Moreover, for $0 < s < 1$, $\frac{dp(s)}{ds}> 0$. 
\end{remark}

\begin{figure}[t!]
    \centering
    \includegraphics[width = 0.4\textwidth]{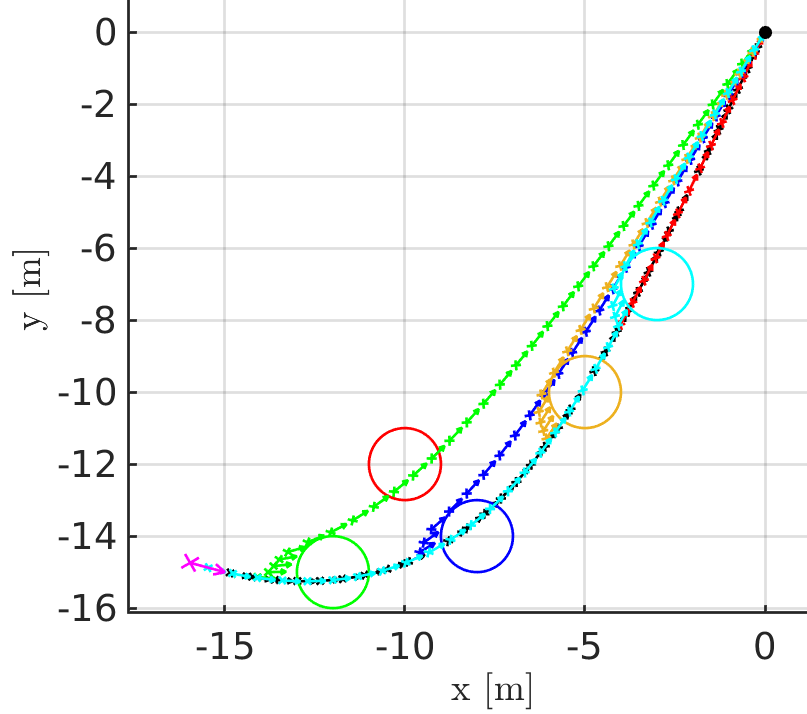}
    \caption{Illustration of how the trajectories vary as a function of different obstacle positions. The target (marked in black) and the robot pose $(-15,-15,-15^\circ)$ are fixed throughout all of the simulations. The different colors correspond to simulations with a single (different) obstacle present at a time. The black trajectory is generated without any obstacle present by only including the CLF constraint in the QP.}
    \label{fig:result1}
\end{figure}

\begin{definition}
    For $\kappa >0$, we define $\sigma_\kappa: \real \to \real$ by 
    \begin{equation}
        \sigma_\kappa(s):= \sigma(\frac{s}{\kappa}).
    \end{equation}
\end{definition}

\begin{proposition} 
\label{prop:Kappas}
Suppose that $\kappa>0$ and $B:{\cal D} \to \real$ is a candidate CBF  with ${\cal D}$ and ${\cal C}$ defined as in \eqref{eq:hGreaterThanZero}. Then $\sigma_\kappa \circ B:{\cal D} \to \real$ is a valid CBF for the system \eqref{eq:ControlSystem} if, and only if, 
$ B: {\cal D} \to \real $ is a valid CBF.  
\end{proposition} 


\begin{proof} 
For $x \in {\cal C}$, $\sigma_\kappa \circ B(x)>0$ and hence satisfies \eqref{eq:CBFrequirement} for $u = 0$. For $ x \in \sim {\cal C}$, by the chain rule and the construction of $\sigma:\real \to \real$,
\begin{equation}
   \frac{\partial \sigma_\kappa \circ B(x) }{\partial x} =  
   \left. \frac{d\sigma(s)}{ds} \right|_{s = \frac{B(x)}{\kappa}} \frac{\partial B(x) }{\partial x} =  \frac{1 }{\kappa} \frac{\partial B(x) }{\partial x}.   
\end{equation}
Hence, the proof of Sect.~\ref{subsec:ProofOfCBF} applies.     
\end{proof}

\begin{figure}[t]
    \centering
    \includegraphics[width = 0.4\textwidth]{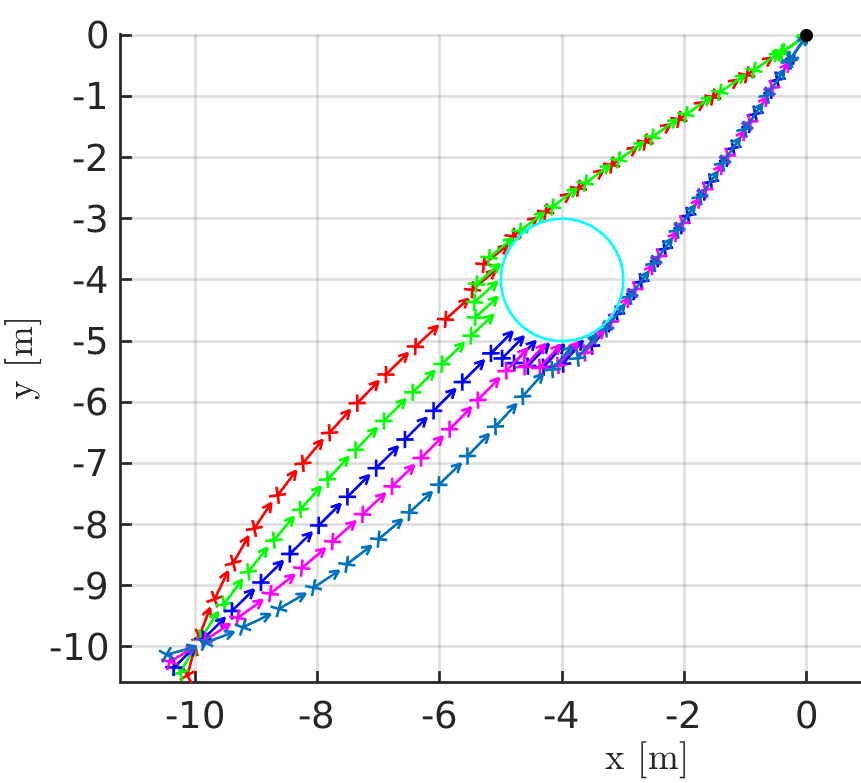}
    \caption{Illustration of how the trajectories vary as a function of different robot orientations with a fixed obstacle location. The target (marked in cyan) and the obstacle at $(-4,-4)$ are fixed throughout all simulations. Each color stands for a different initial robot orientation.}
    \label{fig:result2}
\end{figure}

\begin{proposition} Suppose for $1 \le i \le M$, the CBFs $B_i(x):\real^3 \to \real$ are a positive distance apart. Then there exist $\kappa_1>0$, $\kappa_2>0$, $\ldots$, $\kappa_M>0$, such that for all $i \neq j$,
\begin{equation}
    \label{eq:ChooseKappas}
  \{ x\in \real^3~|~\sigma_{\kappa_i} \circ B_i(x) < 1 \} \cap  \{ x\in \real^3 ~|~  B_j(x) < 0 \} = \emptyset.
\end{equation}    
\end{proposition}
\begin{proof} 
By the disjointness property, $\Delta_i :=\underset{ j \neq i}{\min }~~ \Delta_{ij} >0$. \\

For $S \subset \real^3$ and $x \in \real^3$, define the distance from $x$ to $S$ by
\begin{equation}
  d(x, S):= \inf_{y \in S} ||x-y||.  
\end{equation}
Then,  because (i) $B_i$ is continuous, (ii) the set complement of ${\cal C}_i$ is bounded, and (iii) $d(x, \sim {\cal C}_i)>0 \implies B_i(x)>0$, it follows that
\begin{equation}
    m^\ast_i := \underset{d(x, \sim {\cal C}_i) \le \Delta_i}{\sup} ~B_i(x)
\end{equation}
is a finite positive number. Therefore, for all $0 < \kappa_i <  m^\ast_i$, 
\begin{equation}
   \{ x\in \real^3~|~\sigma_{\kappa_i} \circ B_i(x) < 1 \}  \subset \{x \in \real^3~|~ d(x, \sim {\cal C}_i) \le \Delta_i\}, 
\end{equation}
and hence \eqref{eq:ChooseKappas} holds.
\end{proof}

\begin{figure}[b]
    \centering
    \includegraphics[width = 0.5\textwidth]{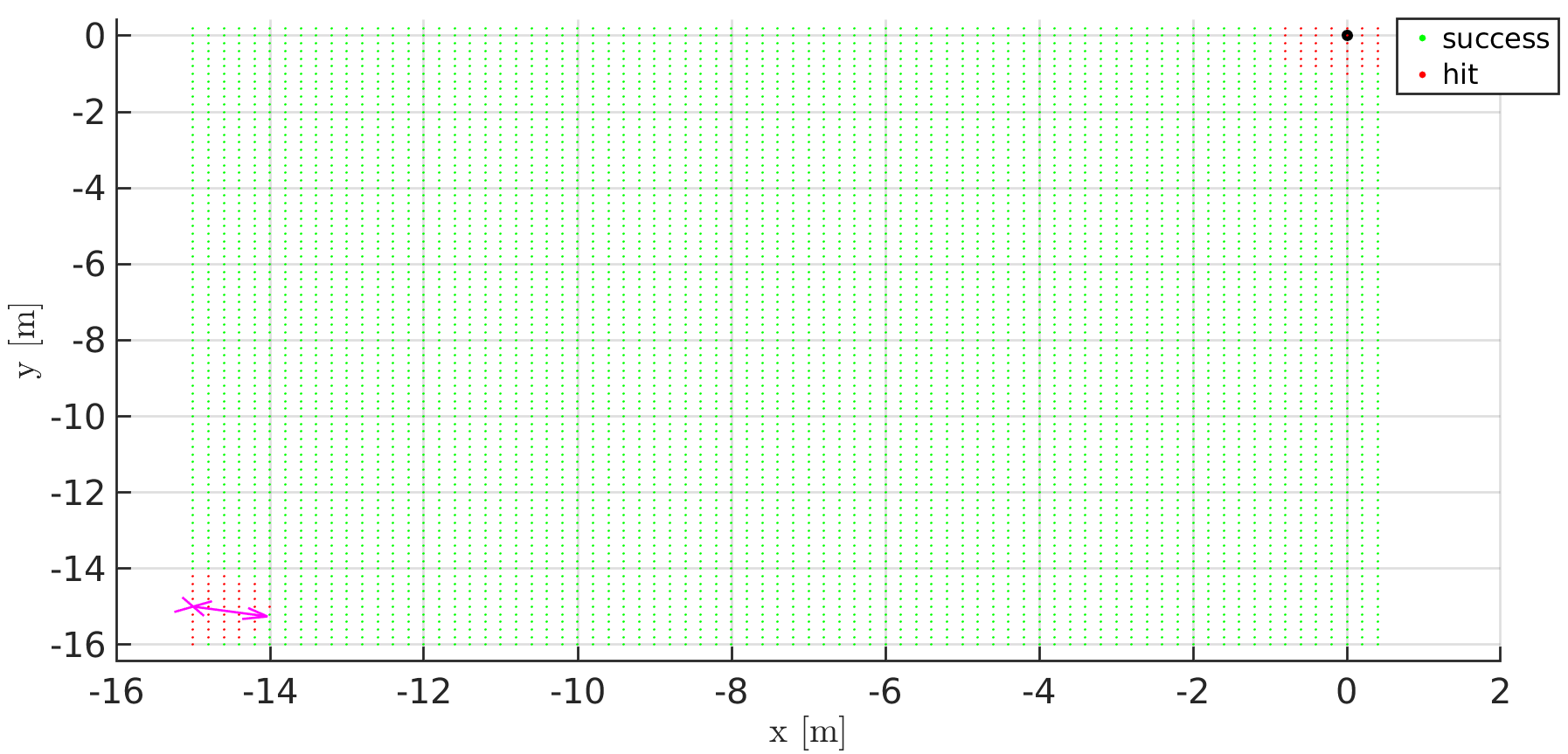}
    \caption{Liveness analysis for the CLF-CBF system. The initial pose is $(-15, -15, -15^\circ)$, and the target is located at $(0,0)$. Each dot in the figure represents the center of an object with radius $(r=1)$. The interval between each center dots are 0.2 meter in both $x$ and $y$ direction. Note that all the red points either originally collide with the robot or the target.}
    \label{fig:ReachabilityAnalysis}
\end{figure}

\begin{figure*}[b]%
    \centering
    \subfloat[]{%
    \label{fig:sim1}%
    \includegraphics[trim=0 0 0 0,clip,width=0.98\columnwidth]{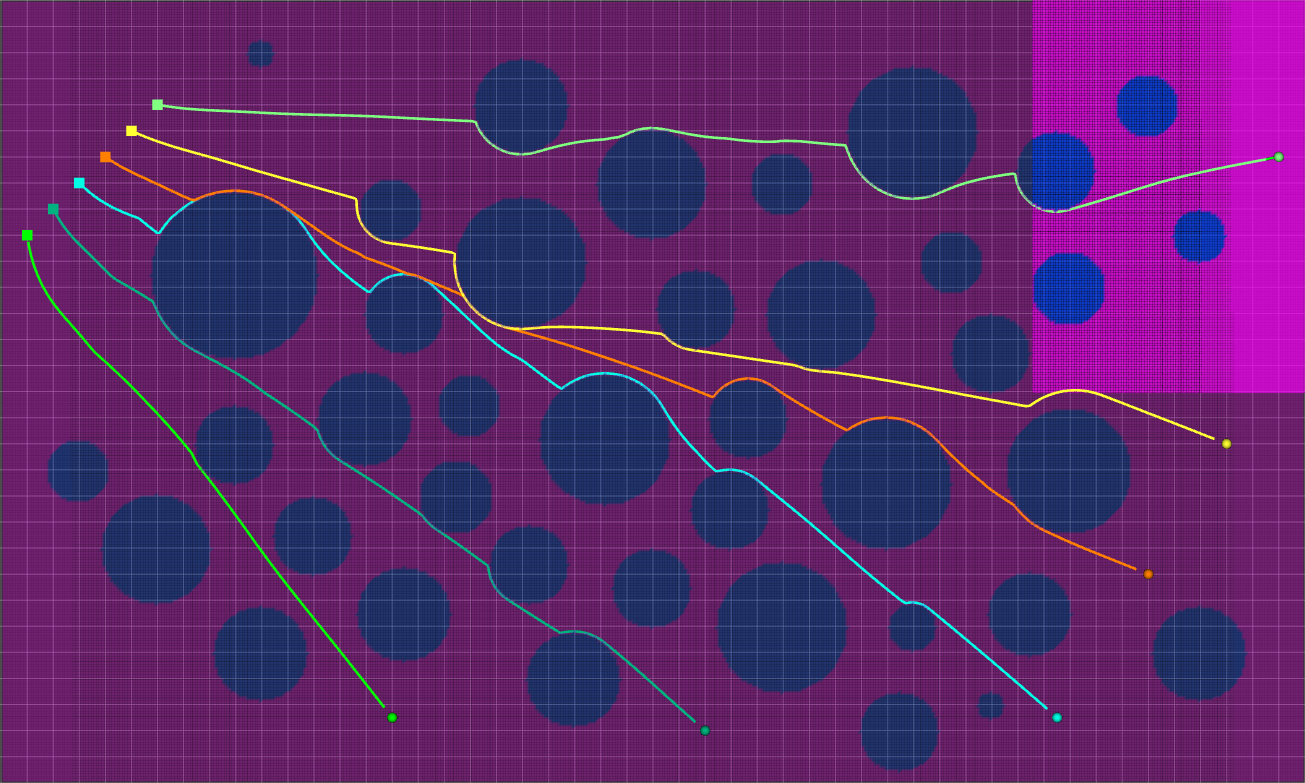}}~
    \subfloat[]{%
    \label{fig:sim2}%
    \includegraphics[trim=0 0 0 0,clip,width=0.98\columnwidth]{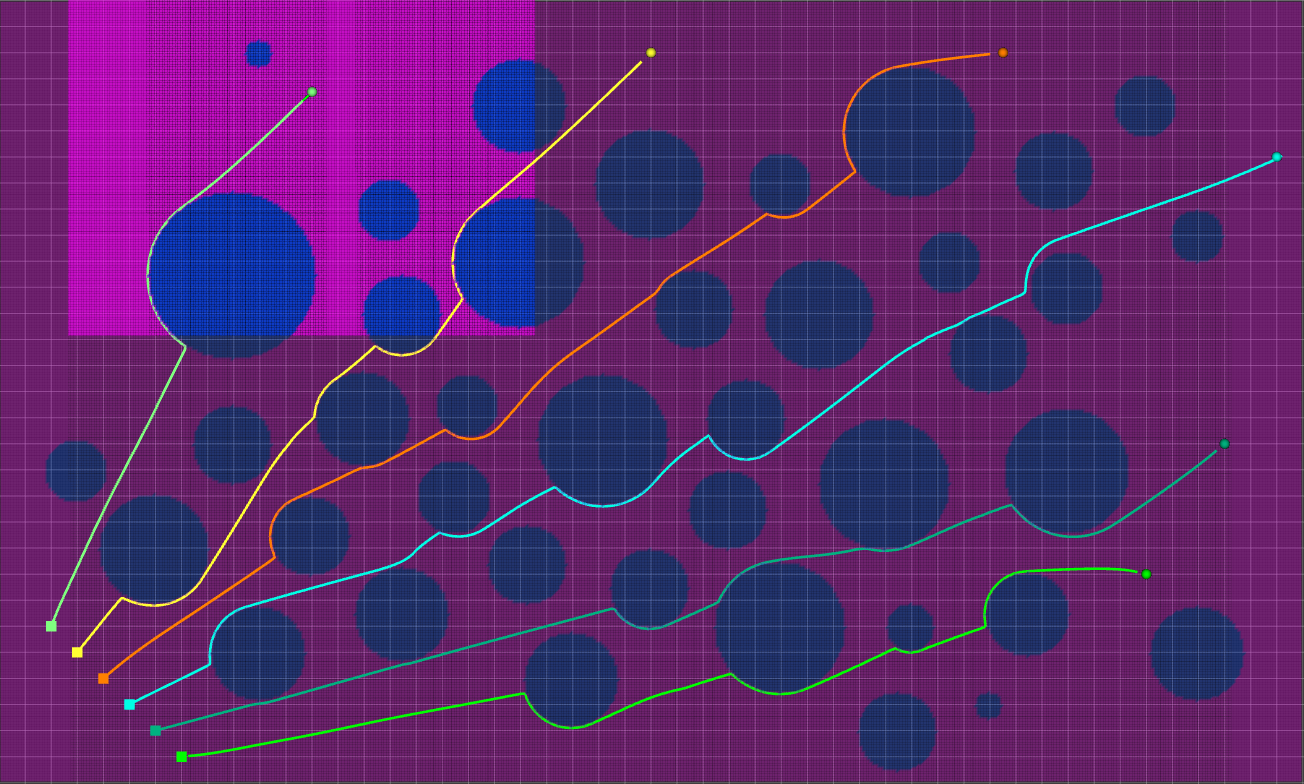}}\\
    \subfloat[]{%
    \label{fig:sim3}%
    \includegraphics[trim=0 0 0 0,clip,width=0.98\columnwidth]{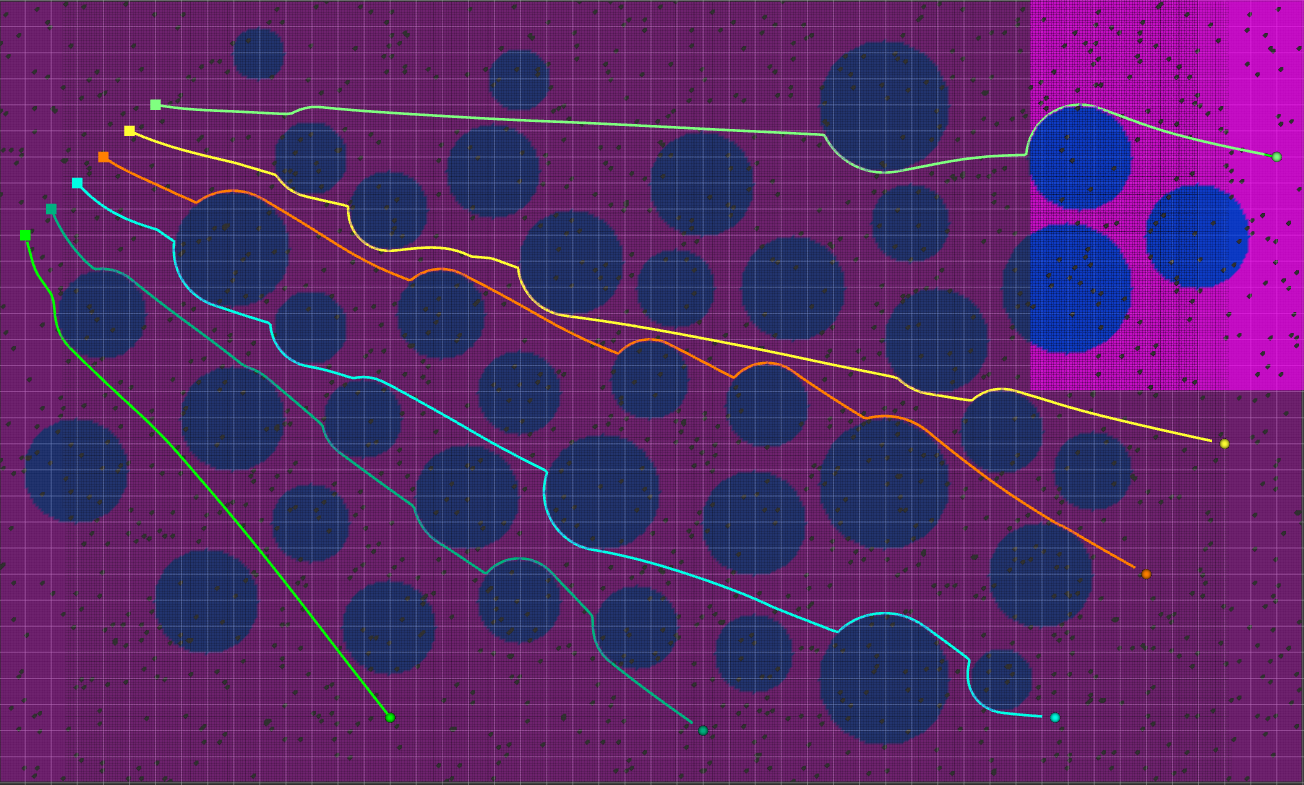}}~
    \subfloat[]{%
    \label{fig:sim4}%
    \includegraphics[trim=0 0 0 0,clip,width=0.98\columnwidth]{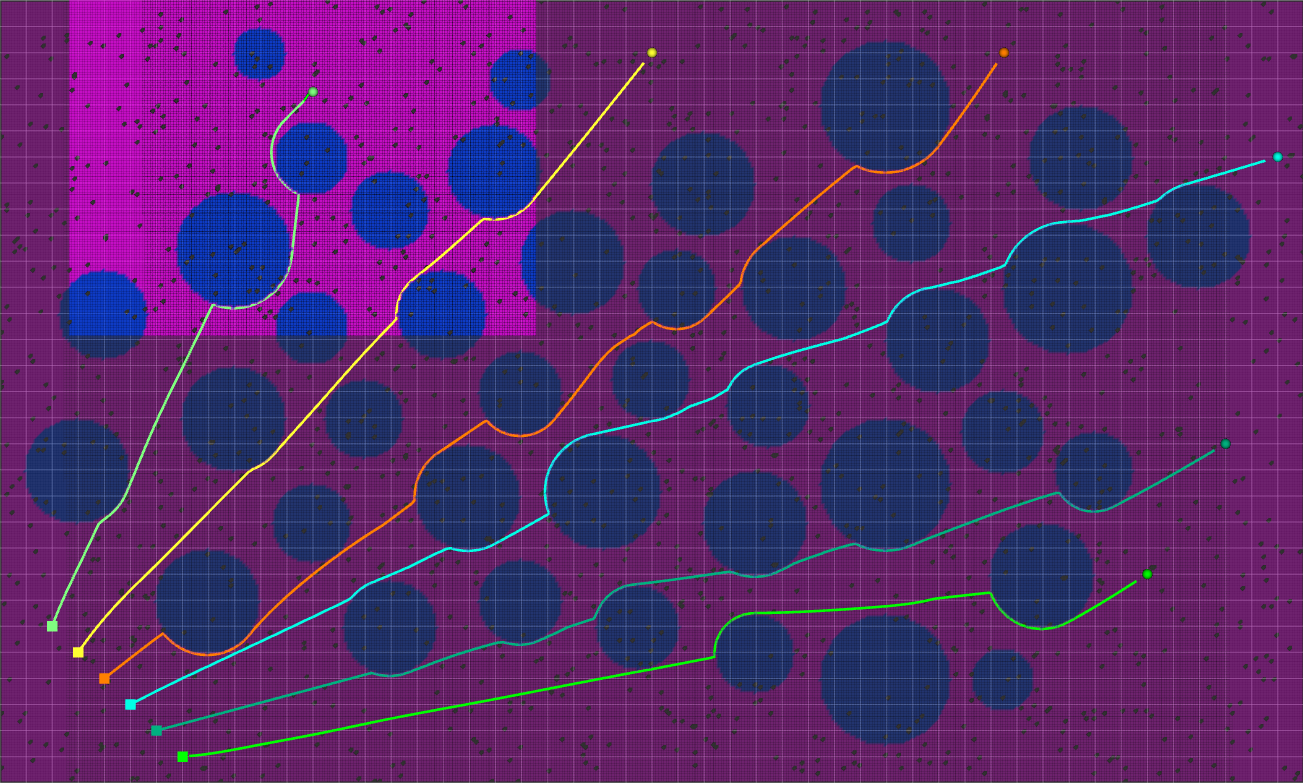}}%
    \caption[]{Robot trajectories with 40 obstacles in the noise-free map 1 (top two) and 40 obstacles in the noisy map 2 (bottom two) synthetic maps with the size of $50\times30$ meters. The highlighted areas are the local map at that specific timestamp. The dark blue circles are the obstacles. Different colors represent different runs in the map.}
    \label{fig:multiple obstacle simulation}
\end{figure*}

\subsection{Multiplication Property of Smooth Saturated CBFs}
 
For $M\ge2$ CBFs corresponding to disjoint obstacles, define the sets 
\begin{equation}
     \label{eq:hGreaterThanZeroMulti}
    \begin{aligned}
       {\cal D}_M &:= \bigcap_{i=1}^{M} {\cal D}_i \\
    {\cal C}_M & :=\bigcap_{i=1}^{M} \{ x \in {\cal D}_M ~|~ B_i(x) \ge 0 \} \\
    & =\bigcap_{i=1}^{M} {\cal C}_i.
    \end{aligned}  
\end{equation}

\begin{theorem}
\label{thm:Product}
Under the assumed disjointness property, the product of smoothly saturated valid CBFs, 
\begin{equation}
    \label{eq:MultiCBF}
    B_M(x) := \prod_{i=1}^{M} \sigma_{\kappa_i} \circ B_i(x),
\end{equation}
is a valid CBF for ${\cal D}_M$, ${\cal C}_M$, and the dynamic system \eqref{eq:ControlSystem}. 
\label{theorem: production property of CBF}
\end{theorem}
\begin{proof}
For $x\in\cal C_M$, the zero control $u\equiv 0$ satisfies \eqref{eq:CBFrequirement} because the drift term $f(x)$ is zero. We show that for $x \not \in {\cal C}_M$, \eqref{eq:CBFrequirement} can be satisfied.

By the disjoint property of the assumed CBF functions, when $B_M(x)<0$, we have $\exists i$, such that $\sigma_{\kappa_i} \circ B_i(x) = B_i(x)<0$, and $\sigma_{\kappa_j} \circ B_j(x)=1$ for $j\neq i$. Hence, $B_M(x)=B_i(x)$. Because $B_i(x)$ is assumed to be a valid CBF function, and both ${\cal D}_M \subset {\cal D}_i$ and ${\cal C}_M \subset {\cal C}_i$ hold, the CBF property holds for $B_M(x)$. 
  \end{proof}  

\begin{remark}
    Due to the way we have constructed the multi-obstacle CBF, the equilibrium analysis for a single obstacle carries over here without changes. This is because, when the robot is at a boundary of an obstacle, the values of the saturated CBFs for the other obstacles will all be one. 
\end{remark}

\section{Smoothing Reference Control Variables}
\label{sec:SmoothRefControl}
We have created a continuous vector field for a single target position surrounded by multiple obstacles. However, discontinuities are introduced when we switch from one target to another. The jumps appear in two ways: when the attraction points used in the CLF are updated and when the associated reference control variable $u^\mathrm{ref}$ is updated. The top row of Fig.~\ref{fig:nonsmooth target switch} illustrates the jumps in the control signals associated with a list of targets (black dots). 

This section proposes a means to smooth the control signals. To begin, we define $\tau : \real \to \real$ by
\begin{equation}
\label{eq:target-switch}
    \tau(s) : = \begin{cases}
            0 & s \le 0 \\
            s^2(3 - 2s) & 0 < s < 1 \\
            1 & s \ge 1,
            \end{cases}
\end{equation}
where we have $\forall s\in\real$, $\frac{d\tau(s)}{ds}$ exists and satisfies
\begin{equation}
\label{eq:target-switch-derivative}
    \frac{d\tau(s)}{ds} : = \begin{cases}
                            0 & s \le 0 \\
                            6s - 6s^2 & 0 < s < 1 \\
                            0 & s \ge 1.
                            \end{cases}
\end{equation}
Upon noting that $\left.\frac{d\tau(s)}{ds}\right|_{s=0}=0$, $\left.\frac{d\tau(s)}{ds}\right|_{s=1}=0$ and for all $0 \leq s \leq 1$, $\frac{d\tau(s)}{ds} > 0$, it follows that $\tau(s)$ is continuously differentiable and monotonic.

Consider now two target positions $\Gcal_1$ and $\Gcal_2$ with their reference control variables $u^\mathrm{ref}_1$ and $u^\mathrm{ref}_2$ from \eqref{eq:CLFSolution}. A smoothed $u^\mathrm{ref}$ is defined as the convex combination
\begin{equation}
    u^\mathrm{ref} = \tau(\frac{t}{T}) u^\mathrm{ref}_1 + (1 - \tau(\frac{t}{T})) u^\mathrm{ref}_2,
    \label{eq:smooth-uref}
\end{equation}
where $T$ is a switching time parameter and $t$ is the estimated time to $\Gcal_1$. By the construction of the scaling function $\tau$ in \eqref{eq:target-switch} and the convex combination in \eqref{eq:smooth-uref}, we only execute the interpolation process if the robot is estimated to be less than $T$ seconds away from $\Gcal_1$. The value of $t$ is updated to the estimated time to $\Gcal_1$ when the $\Gcal_1$ is updated. The value of $t$ decreases while the robot is moving. By applying the smoothed $u^\mathrm{ref}$ and with the same list of goals, the control variables become continuous, as shown in the bottom row of Fig.~\ref{fig:nonsmooth target switch}.


\section{Simulation Results with single and multiple obstacles}
\label{sec:Simulation}
In this section, we first use simulation to study the behavior and liveness of the proposed CLF-CBF system with a single obstacle. Next, we run the system on several synthetic environments with 20 obstacles in Robot Operating System (ROS) \cite{ros} with C++. 

\begin{remark}
    For the CBF in \eqref{eq:OriginalCBF}, we take $Q = I$ and in Prop.~\ref {prop:Kappas}, we take $\kappa_1=\cdots=\kappa_M=\min\{\Delta_{i}^2\}_{i=1}^M$, which is the minimum of the square of the distance between any of the obstacles.
\end{remark}

\subsection{Robot Model in Simulation}
\label{sec:ALIPRobotModel}
In MATLAB and ROS, the bipedal robot is represented by the Angular momentum Linear Inverted Pendulum (ALIP) model~\cite{gong2020angular}. The ALIP robot takes piece-wise constant inputs from the CLF-CBF-QP system. Let $g, H, \tau$ be the gravitational constant, the robot's center of mass height, and the time interval
    of a swing phase, respectively. The motion of an ALIP model on the \xaxis satisfies 
    \begin{equation}
        \begin{bmatrix}
            x_{k+1}\\
            \dot{x}_{k+1}
        \end{bmatrix} = 
        \begin{bmatrix}
            \cosh(\xi) & \frac{1}{\rho}\sinh(\xi)\\
            \rho\sinh(\xi) & \cosh(\xi)
        \end{bmatrix}
        \begin{bmatrix}
            x_{k}\\
            \dot{x}_{k}
        \end{bmatrix} +
        \begin{bmatrix}
            1-\cosh(\xi) \\
            -\rho\sinh(\xi)
        \end{bmatrix} p_x,
    \end{equation}
    where $x_k$ and $\dot{x}_k$ are the contact position and     velocity of the swing foot on the \xaxisN, $p_x$ is the center of mass (CoM) position on the \xaxis of the robot, $\xi =
    \rho\tau$ and $\rho = \sqrt{g/H}$. The motion of the robot on the
    \yaxis can be similarly defined.




\subsection{Behavior Study with Single Obstacle in MATLAB}
The optimal control command of the robot is the solution of the CLF-CBF-QP problem defined in \eqref{eq:CBLF-QP}. The time interval of a swing phase is set to $\tau= 0.3s$. The robot updates its pose based on the ALIP model and the optimal control command. The updated pose is then fed back to the CLF-CBF-QP system to compute the optimal control for the next iteration. This process continues until the robot reaches the target or collides with an obstacle.

Figure~\ref{fig:result1} shows how the trajectories vary as a function of a single obstacle's position with a fixed initial robot pose of $(-15, -15, -15^\circ)$, marked as the magenta arrow. The black trajectory is the nominal trajectory without any obstacles present. Each colored trajectory and matching circle represent a distinct simulation result.
The robot successfully avoids the obstacle in all cases. In Fig.~\ref{fig:result2}, we show how the trajectories vary as a function of different robot orientations with a fixed obstacle location.

\begin{figure}[t]%
    \centering
    \subfloat[]{
        \includegraphics[height=0.7\columnwidth, trim={0cm 0cm 0cm 0cm},clip]{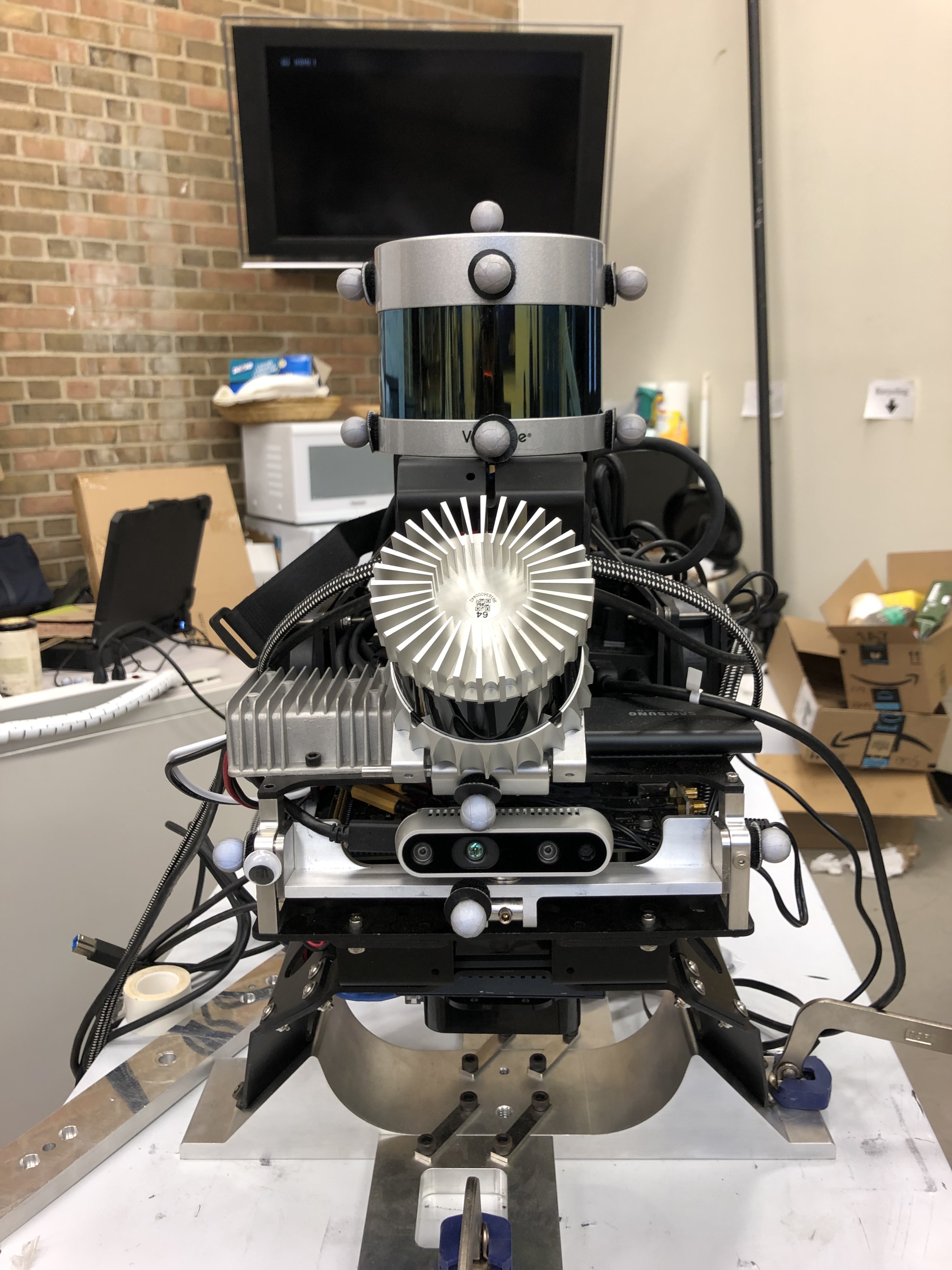}}
    \subfloat[]{
        \includegraphics[height=0.7\columnwidth, trim={15 10 20 30},clip]{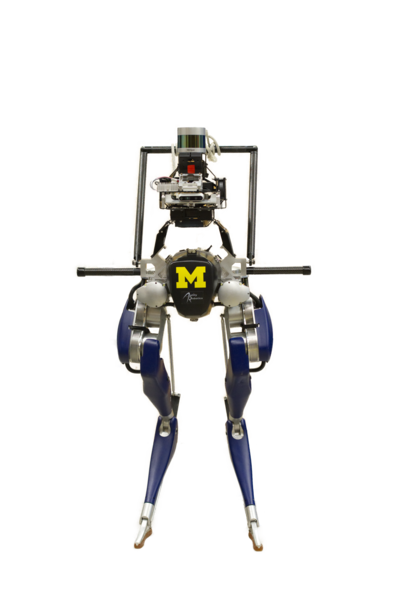}}
\caption[]{The left shows the sensor suite with different sensors, and the right
    shows the sensor suite mounted on Cassie Blue.
}%
    \label{fig:FullTorsoWithCassie}%
\end{figure}

\begin{remark}
When the robot is within an obstacle, there is also a valid solution that pushes the robot outside of the obstacle. Consider the CBF constraint \eqref{eq:QP-cbf-constraint},
\begin{equation}
    L_fB(x)+L_gB(x)u+\eta B(x)\geq 0.
\end{equation}
When the robot is withing an obstacle, $B(x)<0$ and the QP selects $u$ such that $L_fB(x)+L_gB(x)u\geq-\eta B(x)$, causing the robot to leave  the obstacle.
\end{remark}

\begin{figure*}[t]%
    \centering
    \subfloat[]{%
    \label{fig:sim5}%
    \includegraphics[trim=0 0 0 0,clip,width=0.98\columnwidth]{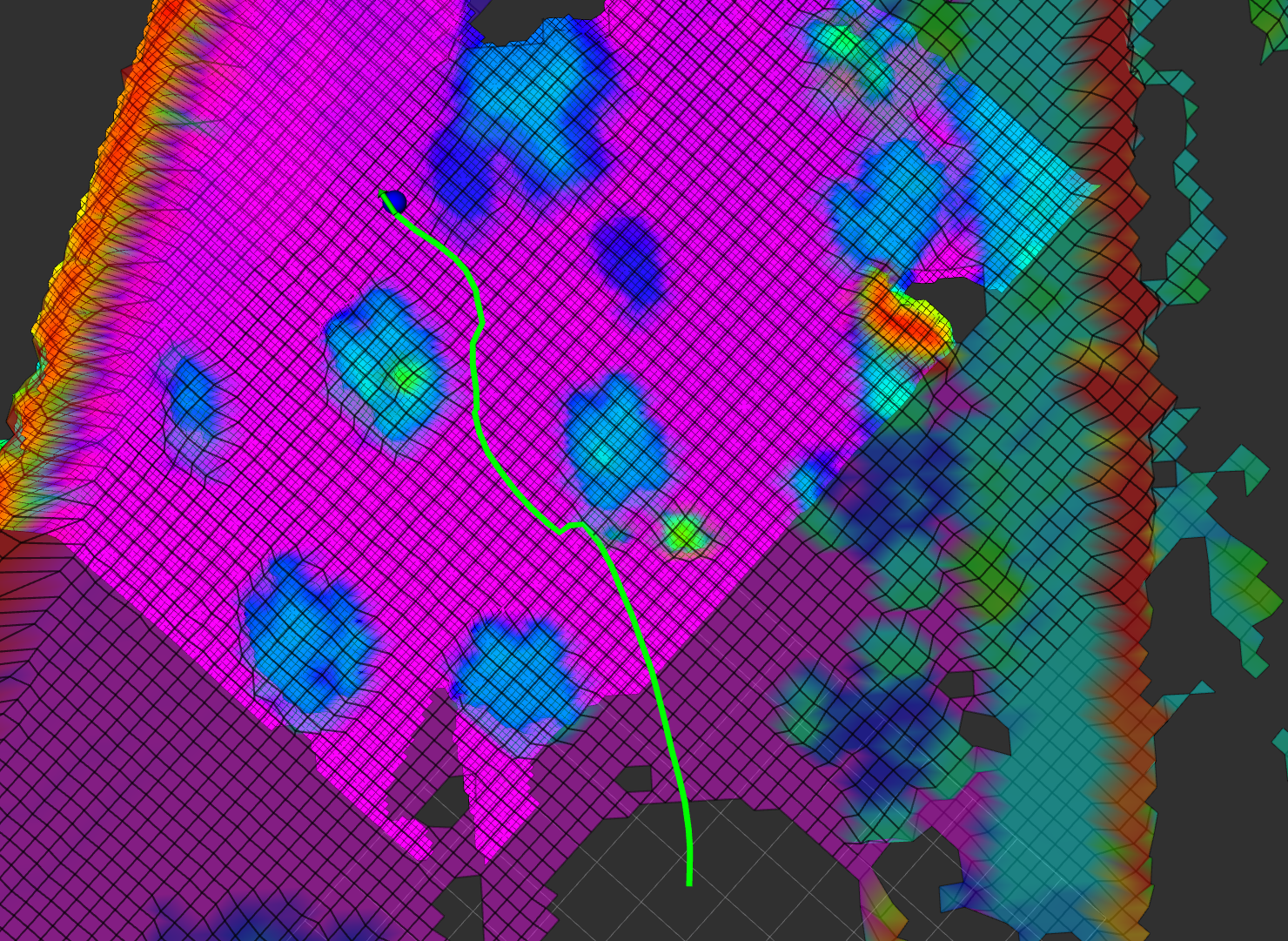}}~
    \subfloat[]{%
    \label{fig:sim6}%
    \includegraphics[trim=0 0 0 0,clip,width=0.98\columnwidth]{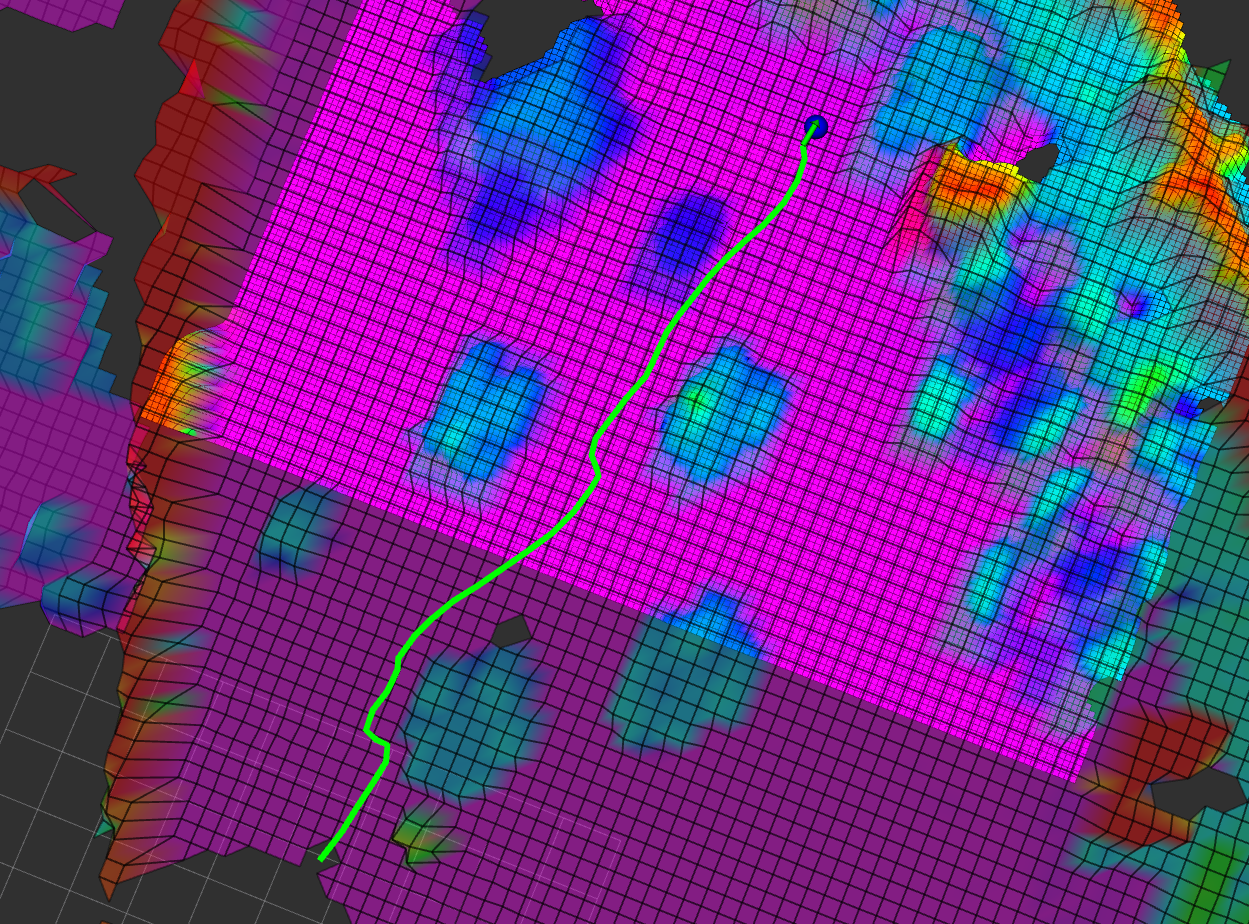}}\\
    \subfloat[]{%
    \label{fig:sim7}%
    \includegraphics[trim=0 0 0 0,clip,width=0.98\columnwidth]{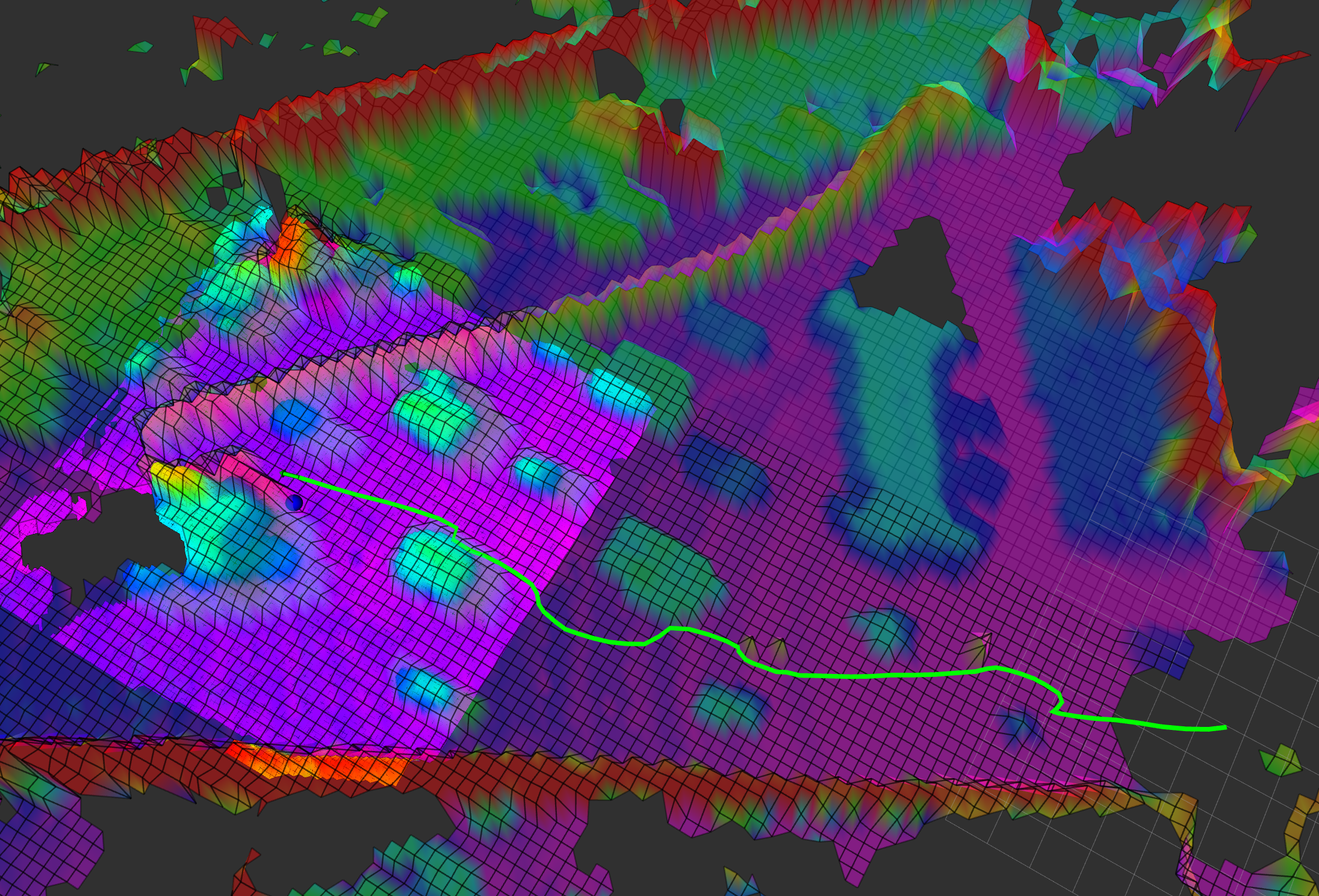}}~
    \subfloat[]{%
    \label{fig:sim8}%
    \includegraphics[trim=0 0 0 0,clip,width=0.98\columnwidth]{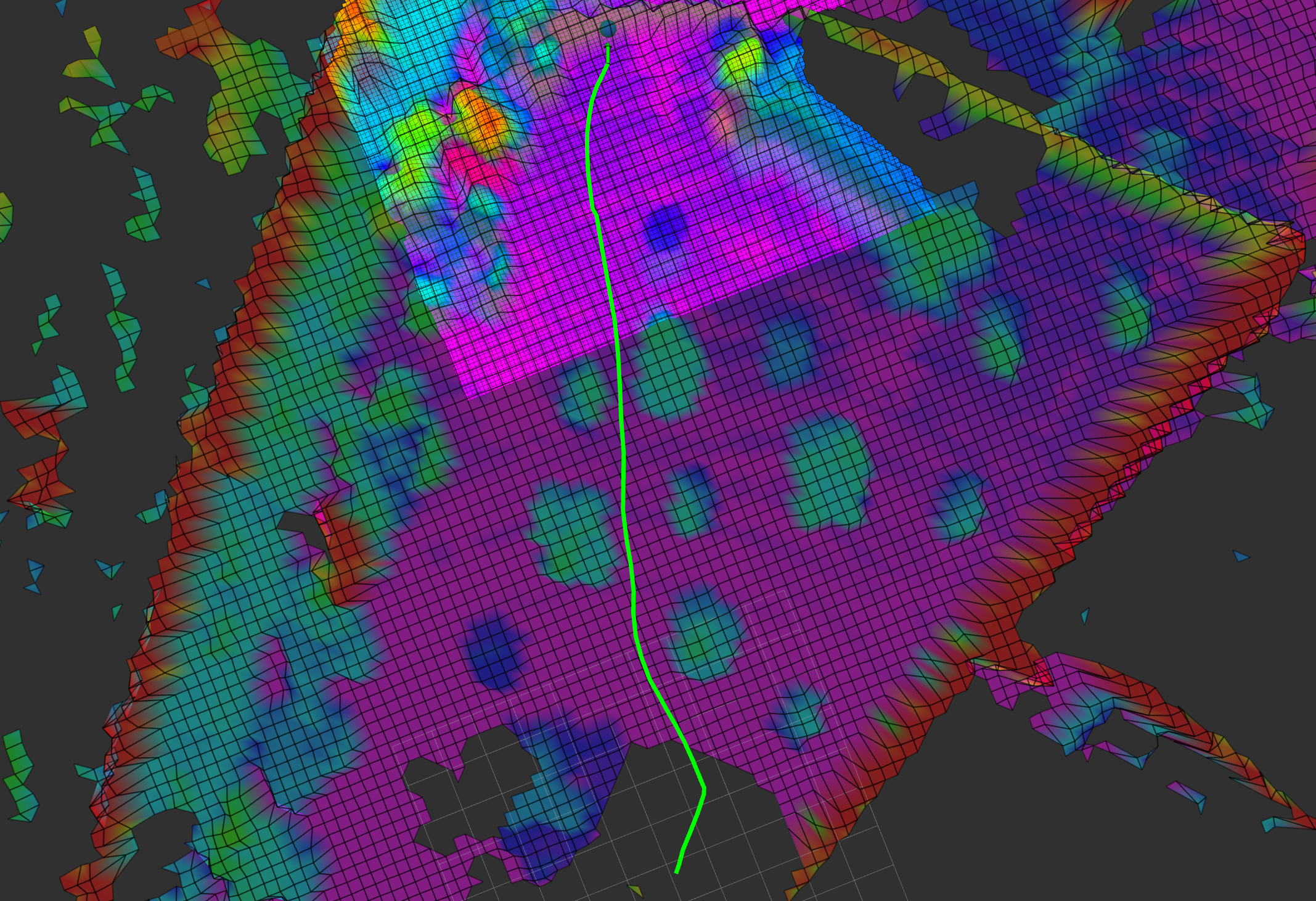}}%
    \caption[]{Autonomy experiments with Cassie Blue on the first floor of FRB. The green arrow is Cassie's pose and the green lines are the resulting trajectories. The blue sphere is the selected target position. The map is colored by height and the highlighted area is the local map.}
    \label{fig:FRB experiments}
\end{figure*}

\subsection{Liveness Analysis in MATLAB}
We analyze the liveness by placing an obstacle with a fixed radius $(r=1)$ at different locations. The robot starts at $(-15, -15, -15^\circ)$ and the target is located at $(0,0)$. The obstacle is placed at every 0.2 meter. If the robot successfully reaches the target without collision, the obstacle location is marked in green otherwise in red, as shown in Fig.~\ref{fig:ReachabilityAnalysis}. All the red points either originally collide with the robot or the target.




\subsection{Multi-Obstacle Simulation with ROS in C++}
In this simulation, we implement a local map centering at robot position with a fixed size and a sub-goal selector to place a target within the local map to achieve long-term planning as not all the obstacles are perceived by the robot at the beginning in practice. Even though the global map is available in simulation but it is not available in practice, therefore, only the information within the local map at the specific timestamp is provided to the robot. The robot model is the same ALIP model in Sec.~\ref{sec:ALIPRobotModel}.

In Fig.~\ref{fig:multiple obstacle simulation}, we generate two noise-free and two noisy synthetic maps with the size of $50\times30$ meters. Each map contains 20 obstacles marked as blue circles. We run six different initial poses and final goals for each map. Different colors represent different runs in the map. The highlighted area is the local map at that specific timestamp. An intermediate goal is chosen at the intersection between the boundary of the local map and the line connecting the robot and the final goal at the current timestamp. If the intermediate goal collides with an obstacle, it is moved back along the line. The intermediate goal is updated when it is reached or becomes inside of an obstacle due to the update of the local map. The robot with ALIP model successfully reaches the goals in all $6\times4 =24$ runs.



\section{Experimental Results on a Bipedal Robot}
\label{sec:Experiments}
We perform several experiments of the proposed CLF-CBF-QP system on Cassie Blue, a bipedal robot with 20 degrees of freedom. The entire system integrates elevation mapping, intermediate goal selection, and the low-level CLF-CBF obstacle avoidance system.

\subsection{Autonomy System Integration}
The following is summarized from \cite{huang2021efficient} for the completeness of the paper. To allow the robot to perceive its surroundings under different lighting conditions
and environments, we designed a perception suite that consists of an RGB-D
camera (Intel RealSense\texttrademark~D435) and a \velodyneN, as shown in Fig.~\ref{fig:FullTorsoWithCassie}. The sensor calibrations are performed
via\cite{
huang2020improvements, huang2021lidartag, huang2020intinsic, huang2021optimal}. The invariant extended Kalman filter (InEKF) \cite{hartley2020contact} estimates the pose of Cassie at 2k Hz. The raw point cloud is motion compensated by the InEKF and then used to build an elevation map. 


\subsection{Autonomy Experiment on Cassie Blue}
We conducted several indoor experiments with Cassie Blue on the first floor of the Ford Robotics Building (FRB) where tables and chairs are considered obstacles. To detect obstacles in the environment, an occupancy grid map is updated in real-time using the timestamped elevation map. Grids with heights greater than 0.2 meters are considered occupied. An occupied grid is defined as the boundary of obstacles if there is an unoccupied grid in its neighborhood. The Breadth First Search (BFS) algorithm \cite{bundy1984breadth} is utilized to find the separated obstacles in the map. Next, we apply the Gift Wrapping Algorithm \cite{cormen2009finding} to the boundary grids of obstacles to find the convex hulls of the obstacles. Finally, the minimum bounding ball algorithm \cite{alzubaidi2014minimum} is applied to the convex hulls to find the minimum bounding circles of the obstacles. The circles are used to represent obstacles in the CBF function \eqref{eq:OriginalCBF}. The target position is selected by clicking a point in the global map. If the final target is not within the current local map, an intermediate goal will be selected within the local map. When an intermediate goal is reached by Cassie or becomes invalid because of the update of the local map, it is updated. In the experiments, Cassie successfully avoids all the obstacles and reaches the target position, as shown in Fig. \ref{fig:FRB experiments}. 
\section{Conclusion}
\label{sec:Conclusion}

This paper presented a reactive planning system that allows a Cassie-series bipedal robot to avoid multiple non-overlapping obstacles via a single, continuously differentiable control barrier function (CBF). The overall system detects an individual obstacle via a height map derived from a LiDAR point cloud and computes an elliptical outer approximation, which is then turned into a quadratic CBF. A continuously differentiable saturation function is presented that preserves the CBF property of a quadratic CBF while allowing the saturated CBFs for individual obstacles to be turned into a single CBF. The CLF-CBF-QP formalism developed by Ames \textit{et al.} can then be applied to ensure that safe trajectories are generated in the presence of multiple obstacles. Liveness is ensured by an analysis of induced equilibrium points that are distinct from the goal state. Safe planning in environments with multiple obstacles is demonstrated both in simulation and experimentally on the Cassie bipedal robot.


\section*{Acknowledgment}
\small{ 
Toyota Research Institute provided funds to support this work. Funding for J. Grizzle was in part provided by NSF Award No.~1808051. This article solely reflects the opinions and conclusions of its
authors and not the funding entities.} 

\bibliographystyle{./DefinesBib/bib_all/elsarticle-num-names}
\bibliography{DefinesBib/bib_all/strings-abrv.bib,DefinesBib/bib_all/ieee-abrv.bib,DefinesBib/bib_all/BipedLab.bib,DefinesBib/bib_all/Books.bib,DefinesBib/bib_all/Bruce.bib,DefinesBib/bib_all/ComputerVision.bib,DefinesBib/bib_all/ComputerVisionNN.bib,DefinesBib/bib_all/IntrinsicCal.bib,DefinesBib/bib_all/L2C.bib,DefinesBib/bib_all/LibsNSoftwares.bib,DefinesBib/bib_all/ML.bib,DefinesBib/bib_all/OptimizationNMath.bib,DefinesBib/bib_all/Other.bib,DefinesBib/bib_all/StateEstimationSLAM.bib,DefinesBib/bib_all/MotionPlanning.bib,DefinesBib/bib_all/Mapping.bib,DefinesBib/bib_all/TrajectoriesOptimization.bib,DefinesBib/bib_all/Controls.bib,DefinesBib/bib_all/CBF.bib, DefinesBib/bib_all/Intro_LR.bib}

\begin{appendix}
    \section{Equilibrium Analysis of Multi-Obstacle Systems}
\label{sec:EquilibriumAnalysis}

We give the complete analysis for a single obstacle, following the work of \cite{reis2020control}. Because the drift term of our model is zero, any equilibrium points are where the optimal control is the zero vector. To avoid this undesirable situation, we seek to find all equilibrium points $\Ecal = \{x | u^*=[0,0,0]^\transpose, r>0, \text{ and }  \delta,\theta \in (-\pi, \pi] \}$, where $u^*$ is the optimal control variable. Recall that$f=\begin{bmatrix} 0 & 0 & 0 \end{bmatrix}^\transpose$ in \eqref{eq:ControlSystem}, which leads to $L_fV(x)=L_fB(x)=0$. We denote the following to re-write the CLF and the CBF constraints:
\begin{equation}
\label{eq:Aelement}
\begin{aligned}
        \begin{bmatrix} d_x & d_y & 0 \end{bmatrix} &= -L_gB(x) \\
        \begin{bmatrix} a_x & a_y & a_\omega \end{bmatrix} &= L_gV(x),
\end{aligned}
\end{equation}
where 
\begin{equation}
\label{eq:LgVelement}
\begin{aligned}
        a_x &= -r\cos(\delta) + \frac{\beta\gamma^2\sin(2\beta\delta)  \sin(\delta)}{2r} \\
        a_y &= -r\sin(\delta) - \frac{\beta\gamma^2\sin(2\beta\delta)  \cos(\delta)}{2r} \\
        a_\omega &= \frac{\beta\gamma\sin(2\beta\delta)}{2}.
\end{aligned}
\end{equation}
The constraints become
\begin{align}
    \mathfrak{L}(x,u,s) &= a_xv_x+a_yv_y+a_\omega\omega-s+\mu V(x),\label{eq:RewriteQPCLFConstraints}\\
    \mathfrak{B}(x,u) &=d_xv_x+d_yv_y-\eta B(x),\label{eq:RewriteQPCBFConstraints}
\end{align}
and the cost function \eqref{eq:QP-cost} of the QP can then be re-written as:
\begin{equation}
\label{eq:RewriteQPCost}
\begin{aligned}
    \mathfrak{J}(u,s)&=\frac{1}{2}h_1(v_x-v_{x}^\mathrm{ref})^2 +\frac{1}{2}h_2(v_y-v_{y}^\mathrm{ref})^2\\
    &+\frac{1}{2}h_3(\omega-\omega^\mathrm{ref})^2+\frac{1}{2}ps^2,
\end{aligned}
\end{equation}
where $\{h_i\}_{i=1}^3$ are the diagonal elements of $H$ in \eqref{eq:QP-cost}, the weights of control variables $[v_x, v_y, \omega]$ with $h_i>0$.

The KKT conditions \cite{gass1997encyclopedia} of this quadratic program are:
\begin{align}
\frac{\partial L}{\partial u}&=Hu^*-Hu^\mathrm{ref}+\lambda_1L_gV^\transpose-\lambda_2L_gB^\transpose=0 \label{eq:KKT-Stationarity-u} \\
\frac{\partial L}{\partial s}&=ps-\lambda_1=0 \label{eq:KKT-Stationarity-s} \\
0&=\lambda_1(L_fV+L_gVu^*+\mu V-s) \label{eq:KKT-slackness-clf} \\
0&=\lambda_2(-L_fB-L_gBu^*-\eta B) \label{eq:KKT-slackness-cbf} \\
0 &\geq L_fV+L_gVu^*+\mu V-s \label{eq:KKT-feasibility-clf} \\
0 &\geq -L_fB-L_gBu^*-\eta B \label{eq:KKT-feasibility-cbf} \\
0 &\leq \lambda_1,\lambda_2 \label{eq:KKT-dual},
\end{align}
where $\lambda_1,\lambda_2\in\mathbb{R}$, and $L$ is the Lagrangian function and defined as
\begin{equation}
L(u,s,\lambda_1,\lambda_2)=\mathfrak{J}(u,s)+\lambda_1\mathfrak{L}(x,u,s)+\lambda_2\mathfrak{B}(x,u).
\end{equation}

Next, we analyze equilibrium points (if any) via four cases depending on whether each CLF or CBF constraint is active or inactive following \cite{reis2020control}.



\subsection{Both CLF and CBF are inactive}
\label{subsec:CLF-inactive-CBF-inactive}
When both constraints are inactive, we have 
\begin{equation}
\label{eq:CLF-inactive-CBF-inactive}
\begin{aligned}
    \lambda_1&=0\\
    \lambda_2&=0\\
    0 &> L_fV+L_gVu^*+\mu V-s\\
    0 &> -L_fB-L_gBu^*-\eta B.
\end{aligned}
\end{equation}
With \eqref{eq:KKT-Stationarity-u} and \eqref{eq:KKT-Stationarity-s}, $u^*$ and $s^*$ in this case are
\begin{equation}
\label{eq:CLF-inactive-CBF-inactive-closed}
\begin{aligned}
   u^*&=u^\mathrm{ref} \\
   s^*&=0.
\end{aligned}
\end{equation}
From \eqref{eq:CLFSolution}, as long as the goal is not reached, $u^\mathrm{ref}$ is not a zero vector. Hence, there is no equilibrium point in this case.

\subsection{CLF constraint inactive and CBF constraint active}
\label{subsec:CLF-inactive-CBF-active}
We prove that there is no equilibrium point in this case by contradiction. When the CLF constraint is inactive and the CBF constraint is active, we have 
\begin{equation}
\label{eq:CLF-inactive-CBF-active}
\begin{aligned}
    \lambda_1&=0\\
    \lambda_2&\geq 0\\
    0 &> L_fV+L_gVu^*+\mu V-s\\
    0 &= -L_fB-L_gBu^*-\eta B.
\end{aligned}
\end{equation}
With \eqref{eq:KKT-Stationarity-u} and \eqref{eq:KKT-Stationarity-s}, $u^*$, $s^*$ and $\lambda_2$ in this case are
\begin{equation}
\label{eq:CLF-inactive-CBF-active-closed}
\begin{aligned}
    u^*&=u^\mathrm{ref} + \lambda_2 H^{-1} L_gB^\transpose \\
    s^*&=0\\
    \lambda_2 &= -\frac{\eta B + L_fB + L_gBu^\mathrm{ref}}{L_gBH^{-1}L_gB^\transpose}.
\end{aligned}
\end{equation}
If there is an equilibrium point, then $u^*$ is the zero vector. Hence, at the equilibrium point, by $L_fV(x)=0$, $u^*=0$ and $s^*=0$, we have
\begin{equation}
\label{eq:CLF-constraint-CLF-inactive}
L_fV+L_gVu^*+\mu V-s^* = \mu V > 0,
\end{equation}
which conflicts with \eqref{eq:CLF-inactive-CBF-active}. Therefore, there is no equilibrium point in this case.

\subsection{CLF constraint active and CBF constraint inactive}
\label{subsec:CLF-active-CBF-inactive}
When the CLF constraint is active and the CBF constraint is inactive, we have 
\begin{equation}
\label{eq:CLF-active-CBF-inactive}
\begin{aligned}
    \lambda_1&\geq 0\\
    \lambda_2&= 0\\
    0 &= L_fV+L_gVu^*+\mu V-s\\
    0 &> -L_fB-L_gBu^*-\eta B.
\end{aligned}
\end{equation}
With \eqref{eq:KKT-Stationarity-u} and \eqref{eq:KKT-Stationarity-s}, $u^*$, $s^*$ and $\lambda_1$ in this case are
\begin{equation}
\label{eq:CLF-active-CBF-inactive-closed}
\begin{aligned}
    u^*&=u^\mathrm{ref} - \lambda_1 H^{-1} L_gV^\transpose \\
    s^*&=\frac{\lambda_1}{p}\\
    \lambda_1 &= \frac{p\mu V + pL_fV + pL_gVu^\mathrm{ref}}{pL_gVH^{-1}L_gV^\transpose+1}.
\end{aligned}
\end{equation}
Using the variables defined in \eqref{eq:Aelement}, $u^*$ can be rewritten as:
\begin{equation}
\label{eq:CLF-active-CBF-inactive-u}
u^*=\begin{bmatrix} v_x^* \\ v_y^* \\ \omega^* \end{bmatrix} = \begin{bmatrix} v_x^\mathrm{ref} - \frac{\lambda_1a_x}{h_1} \\ v_y^\mathrm{ref} - \frac{\lambda_1a_y}{h_2} \\ \omega^\mathrm{ref} - \frac{\lambda_1a_\omega}{h_3} \end{bmatrix}.
\end{equation}
We know from \eqref{eq:LgVelement} that 
\begin{equation}
\label{eq:ay-aomega-delta}
    (a_y=0~~\& ~~a_\omega=0) \Longleftrightarrow \delta=0.
\end{equation}
In addition, we know from \eqref{eq:CLFSolution} that 
\begin{equation}
\label{eq:vy-omega-delta}
    (v_{y}^\mathrm{ref}=0 ~~\& ~~\omega^\mathrm{ref}=0) \Longleftrightarrow \delta=0.
\end{equation}
Therefore, we split this case into three cases based on the value of $\delta$.

\subsubsection{\texorpdfstring{$\delta=0$}{Lg} (Case I)} 
\label{subsubsec:CLF-active-CBF-inactive-case1}
Substituting $\delta=0$ to \eqref{eq:Aelement}, we have $a_x=-r<0,a_y=0,a_\omega=0$, and to \eqref{eq:CLFSolution}, we have $v_{x}^\mathrm{ref}>0, v_{y}^\mathrm{ref}=0,\omega^\mathrm{ref}=0$. Finally, with \eqref{eq:KKT-dual}, the optimal control command \eqref{eq:CLF-active-CBF-inactive-u} can be simplified as:
\begin{equation}
\label{eq:CLF-active-CBF-inactive-u-case1}
u^*=\begin{bmatrix} v_x^* \\ v_y^* \\ \omega^* \end{bmatrix} = \begin{bmatrix} v_x^\mathrm{ref} + \frac{\lambda_1r}{h_1} > 0 \\ 0 \\ 0 \end{bmatrix}.
\end{equation}
The optimal control command is not a zero vector, and hence there is no equilibrium point in this case.

\subsubsection{\texorpdfstring{$\delta>0$}{Lg} (Case II)} 
\label{subsubsec:CLF-active-CBF-inactive-case2}
When $\delta>0$, by the definitions in \eqref{eq:LgVelement}, we have $a_y<0,a_\omega>0$, and by \eqref{eq:CLFSolution}, we have $v_{y}^\mathrm{ref}>0,\omega^\mathrm{ref}<0$. With \eqref{eq:KKT-dual} and \eqref{eq:CLF-active-CBF-inactive-u}, we have
\begin{equation}
\label{eq:CLF-active-CBF-inactive-u-case2}
\begin{aligned}
v_y^* &= v_y^\mathrm{ref} - \frac{\lambda_1a_y}{h_2} > 0\\
\omega^* &= \omega^\mathrm{ref} - \frac{\lambda_1a_\omega}{h_3} < 0.
\end{aligned}
\end{equation}
The optimal control command is not a zero vector in this case. Therefore, there is no equilibrium points in this case either.

\subsubsection{\texorpdfstring{$\delta<0$}{Lg} (Case III)} 
\label{subsubsec:CLF-active-CBF-inactive-case3}
Similarly, by \eqref{eq:Aelement} and \eqref{eq:CLFSolution}, we have $a_y>0,a_\omega<0$ and $v_{y}^\mathrm{ref}<0,\omega^\mathrm{ref}>0$. With \eqref{eq:KKT-dual} and \eqref{eq:CLF-active-CBF-inactive-u}, we have
\begin{equation}
\label{eq:CLF-active-CBF-inactive-u-case3}
\begin{aligned}
v_y^* &= v_y^\mathrm{ref} - \frac{\lambda_1a_y}{h_2} < 0\\
\omega^* &= \omega^\mathrm{ref} - \frac{\lambda_1a_\omega}{h_3} > 0
\end{aligned}
\end{equation}
The optimal control command is not a zero vector in this case; hence, there is no equilibrium point in this case.

In summary, there is no equilibrium point when the CLF constraint is active and the CBF constraint is inactive.

\subsection{Both CLF and CBF constraint are active}
\label{subsec:CLF-active-CBF-active}

When the CLF constraint is active and the CBF constraint is active, we have 
\begin{equation}
\label{eq:CLF-active-CBF-active}
\begin{aligned}
    \lambda_1&\geq 0\\
    \lambda_2&\geq 0\\
    0 &= L_fV+L_gVu^*+\mu V-s\\
    0 &= -L_fB-L_gBu^*-\eta B.
\end{aligned}
\end{equation}
We can rewrite \eqref{eq:KKT-Stationarity-u} and \eqref{eq:KKT-Stationarity-s} as:
\begin{equation}
\label{eq:CLF-active-CBF-active-closed}
\begin{aligned}
    u^*&=u^\mathrm{ref} - \lambda_1 H^{-1} L_gV^\transpose + \lambda_2 H^{-1} L_gB^\transpose \\
    s^*&=\frac{\lambda_1}{p}\\
\end{aligned}
\end{equation}
Using the variables defined in \eqref{eq:Aelement}, $u^*$ can be rewritten as:
\begin{equation}
\label{eq:CLF-active-CBF-active-u}
u^*=\begin{bmatrix} v_x^* \\ v_y^* \\ \omega^* \end{bmatrix} = 
\begin{bmatrix} 
    v_x^\mathrm{ref} - \frac{\lambda_1a_x}{h_1} - \frac{\lambda_2d_x}{h_1} \\ 
    v_y^\mathrm{ref} - \frac{\lambda_1a_y}{h_2} - \frac{\lambda_2d_y}{h_2} \\ 
    \omega^\mathrm{ref} - \frac{\lambda_1a_\omega}{h_3} 
\end{bmatrix}.
\end{equation}

When the robot is at an equilibrium point, $u^*$ is the zero vector. By \eqref{eq:CLF-active-CBF-active} and $L_fB=0$, $u^*=0$, we have $B=0$, which implies that the robot is at the boundary of an obstacle. In the following proof of Sec. \ref{subsec:CLF-active-CBF-active}, we will assume the robot is at the boundary of obstacles.

The property of $B = 0$ leads to an immediate proposition which is helpful in finding the equilibrium point in the system when one of the components of the optimal control is $0$.
\begin{proposition}
\label{prop:ObstacleBoundary}
$d_y=0 \implies v_x^*=0$.
\end{proposition}
\begin{proof}
By the proof in \ref{subsec:ProofOfCBF}, we have $L_gB(x)=\nabla B(x)\cdot g(x)\neq 0$ for $x\in{\cal D}$. Therefore, when $d_y=0$, we have $d_x \neq 0$. Then, we can further have $L_gB(x)u^*=0 \implies v_x^*=0$. 
\end{proof}


In addition, with the properties \eqref{eq:ay-aomega-delta} and \eqref{eq:vy-omega-delta}, we split this case into four cases based on whether $\delta$ and $d_y$ are zero.

\subsubsection{\texorpdfstring{$d_y=\delta=0$}{Lg} (Case I)} 
\label{subsubsec:CLF-active-CBF-active-case1}
Substituting to \eqref{eq:Aelement}, we have $a_x=-r<0,a_y=0,a_\omega=0$, and to \eqref{eq:CLFSolution}, we have $v_{x}^\mathrm{ref}>0,v_{y}^\mathrm{ref}=0,\omega^\mathrm{ref}=0$. Finally, with Proposition \ref{prop:ObstacleBoundary}, in this case the optimal control command \eqref{eq:CLF-active-CBF-active-u} can be written as:
\begin{equation}
\label{eq:CLF-active-CBF-active-u-case1}
u^*=\begin{bmatrix} v_x^* \\ v_y^* \\ \omega^* \end{bmatrix} = \begin{bmatrix} v_x^\mathrm{ref} - \frac{\lambda_1a_x}{h_1} - \frac{\lambda_2d_x}{h_1} \\ 0 \\ 0 \end{bmatrix} = \begin{bmatrix} 0 \\ 0 \\ 0 \end{bmatrix}.
\end{equation}
$\lambda_1$ and $\lambda_2$ can be obtained by \eqref{eq:CLF-active-CBF-active-u-case1}, \eqref{eq:CLF-active-CBF-active} and \eqref{eq:CLF-active-CBF-active-closed}:
\begin{equation}
\label{eq:CLF-active-CBF-active-mu-case1}
\begin{aligned}
\lambda_1 &= p\mu V > 0\\
\lambda_2 &= \frac{h_1v_x^\mathrm{ref}-p\mu Va_x}{d_x}
\end{aligned}
\end{equation}
By \eqref{eq:KKT-dual} and \eqref{eq:CLF-active-CBF-active-mu-case1}, we have
\begin{equation}
\label{eq:CLF-active-CBF-active-dx-case1}
\because v_x^\mathrm{ref}>0, a_x<0, \frac{h_1v_x^\mathrm{ref}-p\mu Va_x}{d_x}\geq0 \longrightarrow d_x>0. 
\end{equation}
Hence, there is an equilibrium point when $B=0$, $d_y=\delta=0$ and $d_x>0$.

\subsubsection{\texorpdfstring{$d_y\neq0,\delta=0$}{Lg} (Case II)} 
\label{subsubsec:CLF-active-CBF-active-case2}
When $\delta=0$, by \eqref{eq:Aelement} and \eqref{eq:CLFSolution}, we have $a_x=-r<0,a_y=0,a_\omega=0$ and $v_{x}^\mathrm{ref}>0,v_{y}^\mathrm{ref}=0,\omega^\mathrm{ref}=0$. Finally, with \eqref{eq:KKT-dual}, the optimal control command \eqref{eq:CLF-active-CBF-active-u} can be simplified as:
\begin{equation}
\label{eq:CLF-active-CBF-active-u-case2}
u^*=\begin{bmatrix} v_x^* \\ v_y^* \\ \omega^* \end{bmatrix} = \begin{bmatrix} v_x^\mathrm{ref} - \frac{\lambda_1a_x}{h_1} - \frac{\lambda_2d_x}{h_1} \\ -\frac{\lambda_2d_y}{h_2}\neq0 \\ 0 \end{bmatrix}.
\end{equation}
Because $v_y^*\neq0$, the optimal command is not a zero vector in this case. Equilibrium points don't exist when $d_y\neq$ and $\delta=0$.

\subsubsection{\texorpdfstring{$\delta>0$}{Lg} (Case III)}
\label{subsubsec:CLF-active-CBF-active-case3}
When $\delta>0$, by \eqref{eq:CLF-active-CBF-inactive-u-case2}, $\omega^*<0$. Hence, the optimal command is not a zero vector and there are no equilibrium points in this case.

\subsubsection{\texorpdfstring{$\delta<0$}{Lg} (Case IV)} 
\label{subsubsec:CLF-active-CBF-active-case4}
When $\delta<0$, by \eqref{eq:CLF-active-CBF-inactive-u-case3}, $\omega^*>0$. Hence, the optimal command is not a zero vector and there are no equilibrium points in this case.

\end{appendix}

\end{document}